\documentclass[final]{cvpr}
\usepackage{color}
\usepackage[utf8]{inputenc} % allow utf-8 input
\usepackage[T1]{fontenc}    % use 8-bit T1 fonts
\usepackage{url}            % simple URL typesetting
\usepackage{booktabs}       % professional-quality tables
\usepackage{amsfonts}       % blackboard math symbols
\usepackage{nicefrac}       % compact symbols for 1/2, etc.
\usepackage{microtype}      % microtypography
\usepackage{mathtools}
\usepackage{multirow}
\usepackage{subfigure}
\usepackage{makecell}
\usepackage{wrapfig}
\usepackage{times}
\usepackage{epsfig}
\usepackage{graphicx}
\usepackage{amsmath}
\usepackage{amssymb}
\usepackage[numbers]{natbib}
\usepackage{algorithm}
\usepackage{algorithmic}
\usepackage{nopageno}

\usepackage{amsthm}
\newtheorem{theorem}{Theorem}%[section]

\newtheorem{corollary}[theorem]{Corollary}

\newtheoremstyle{TheoremRep}
        {\topsep}{\topsep}              %%% space between body and thm
        {\itshape}                      %%% Thm body font
        {}                              %%% Indent amount (empty = no indent)
        {\bfseries}                     %%% Thm head font
        {.}                             %%% Punctuation after thm head
        { }                             %%% Space after thm head
        {\thmname{#1}\thmnote{ \bfseries #3}}%%% Thm head spec
\theoremstyle{TheoremRep}
\newtheorem{theoremrep}{Theorem}

\newcommand{\STAB}[1]{\begin{tabular}{@{}c@{}}#1\end{tabular}}

\usepackage[pagebackref=true,breaklinks=true,colorlinks,bookmarks=false]{hyperref}

\def\cvprPaperID{5952} % *** Enter the CVPR Paper ID here
\def\confYear{CVPR 2021}
\title{How Robust are Randomized Smoothing based Defenses to Data Poisoning?}
\author{Akshay Mehra\textsuperscript{1}, Bhavya Kailkhura\textsuperscript{2}, Pin-Yu Chen\textsuperscript{3} and Jihun Hamm\textsuperscript{1}\\
{\small \textsuperscript{1}Tulane University \quad \textsuperscript{2}Lawrence Livermore National Laboratory \quad \textsuperscript{3} IBM Research}\\ 
{\tt\small\{amehra, jhamm3\}@tulane.edu, kailkhura1@llnl.gov, pin-yu.chen@ibm.com}\\
}
\if0
\author{Akshay Mehra\thanks{Tulane University}\\
{\tt\small amehra@tulane.edu}
\and

Bhavya Kailkhura\thanks{Lawrence Livermore National Laboratory}\\
{\tt\small kailkhura1@llnl.gov}

\and

Pin-Yu Chen\thanks{IBM Research}\\
{\tt\small pin-yu.chen@ibm.com}

\and

Jihun Hamm\footnotemark[1]\\
{\tt\small jhamm3@tulane.edu}

}
\fi
\begin{document}
\maketitle
%%%%%%%%% ABSTRACT
\begin{abstract}
Predictions of certifiably robust classifiers remain constant in a neighborhood of a point, making them resilient to test-time attacks with a guarantee. 
In this work, we present a previously unrecognized threat to robust machine learning models that highlights the importance of training-data quality in achieving high certified adversarial robustness.
Specifically, we propose a novel bilevel optimization based data poisoning attack that degrades the robustness guarantees of certifiably robust classifiers.
Unlike other poisoning attacks that reduce the accuracy of the poisoned models on a small set of target points, our attack reduces the average certified radius (ACR) of an entire target class in the dataset. 
Moreover, our attack is effective even when the victim trains the models from scratch using state-of-the-art robust training methods such as Gaussian data augmentation\cite{cohen2019certified}, MACER\cite{zhai2020macer}, and SmoothAdv\cite{salman2019provably} that achieve high certified adversarial robustness.
To make the attack harder to detect, we use clean-label poisoning points with imperceptible distortions. 
The effectiveness of the proposed method is evaluated by poisoning MNIST and CIFAR10 datasets and training deep neural networks using previously mentioned training methods and certifying the robustness with randomized smoothing. 
The ACR of the target class, for models trained on generated poison data, can be reduced by more than 30\%. Moreover, the poisoned data is transferable to models trained with different training methods and models with different architectures.

\end{abstract}

\vspace{-0.4cm}
\section{Introduction}
Data poisoning \cite{biggio2012poisoning,jagielski2018manipulating,shafahi2018poison,steinhardt2017certified,zhu2019transferable} 
%\JH{Parenthesis is used here but not everywhere. I would remove it everywhere.} 
is a training-time attack where the attacker is assumed to have access to the training data on which the victim will train the model.
The attacker can modify the training data 
%by either injecting poison samples or modifying the existing data 
in a manner that the model trained on this poisoned data performs as the attacker desires. 
The data hungry nature of modern machine learning methods make them vulnerable to poisoning attacks. Attackers can place the poisoned data online and wait for it to be scraped by victims trying to increase the size for their training sets. 
Another easy target for data poisoning is data collection by crowd sourcing where malicious users can corrupt the data they contribute. In most cases, an attacker can modify only certain parts of the training data such as change the features or labels for a specific class or modify a small subset of the data from all classes.
In this work, we assume the attacker wants to affect the performance of the victim's models on a target class and modifies only the features of the points belonging to that class (without affecting the labels). To evade detection, the attacker is constrained to only add imperceptibly small perturbations to the points of the target class. 
Many previous works \cite{munoz2017towards,shafahi2018poison,huang2020metapoison,koh2017understanding,zhu2019transferable,chen2017targeted,ji2017backdoor,turner2018clean} have shown the effectiveness of poisoning in affecting the accuracy of models trained on poisoned data compared to the accuracy achievable by training with clean data. In most works, the victim is assumed to use standard training by minimizing the empirical loss on the poisoned data to train the models and thus the attack is optimized to hurt the accuracy of standard training. 
However, recent research on test-time evasion attacks \cite{carlini2017adversarial,athalye2018obfuscated,uesato2018adversarial,bulusu2020anomalous} suggests that models trained with standard training are not robust to adversarial examples, making the assumption of victim relying on standard training to train the models for deployment questionable. 
%on data poisoning  have focused on using poisoning to reduce the accuracy of the poisoned model on clean data. 
%(\JH{I changed `hurt' to `reduce' everywhere}) okay
%However, these attacks have only succeeded on a very small subsets of target points, sometimes as small as a single point, especially when the victim is assumed to use a deep neural network for training on the poisoned data.
%Moreover, the effectiveness of the attacks drops significantly in the setting where the victim could retrain the entire network from scratch \cite{shafahi2018poison,zhu2019transferable} or use data augmentation during training \cite{schwarzschild2020just}. 
%This shows the difficulty of data poisoning attacks in reducing the generalization performance of machine learning models especially those of deep neural networks.

\begin{figure*}[tb]
  \centering{\includegraphics[width=0.99\textwidth]{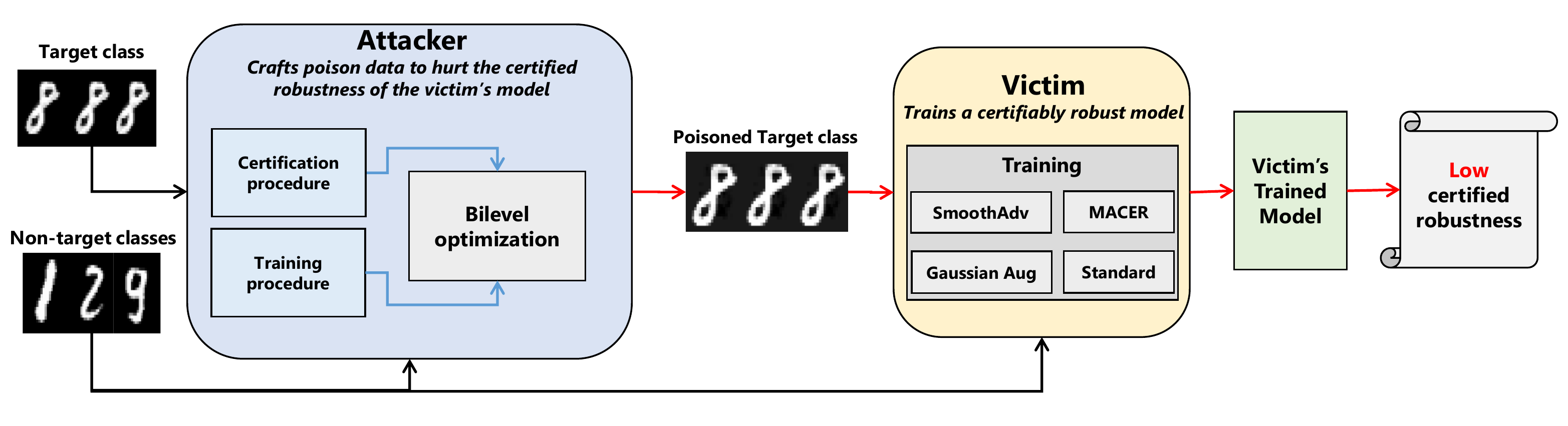}}
  \caption{Overview of our poisoning against certified defenses (PACD) attack which generates poisoned data to reduce the certified robustness of the victim's model trained with methods such as Gaussian data augmentation(GA)\cite{cohen2019certified}, SmoothAdv\cite{salman2019provably} and MACER\cite{zhai2020macer} on a target class.
  %\JH{I like this figure, but what's the difference between circles and squares?}
  %\JH{I would remove ``victim's model architecture in the figure.}
  %\JH{This figure is not referred to in the paper!}
  }
  \label{fig:overview}
\end{figure*}

Thus, in a realistic scenario, where the aim of the victim is to deploy the model, it's better to assume that the victim will rely on training procedures that yield classifiers which are provably robust to test-time attacks.
%However, recent research on test-time evasion attacks \cite{carlini2017adversarial,athalye2018obfuscated,uesato2018adversarial,bulusu2020anomalous} suggests that one should measure the robustness of a model to adversarial examples before it is deployed. 
Several recent works have proposed methods for training certifiably robust models whose predictions are guaranteed to be constant in a neighbourhood of a point. However, many of these methods \cite{raghunathan2018semidefinite,gowal2018effectiveness,huang2019achieving,xu2020automatic} do not scale to deep neural networks or large datasets, due to their high complexity. Moreover, the effect of training data quality on the performance of these certified defenses at test time remains largely unexplored.
%\PY{Moreover, how training-data quality affect the performance of certified methods at test-time remains largely unexplored.}
Recently, randomized smoothing (RS) based certification methods \cite{lecuyer2019certified,li2019certified,cohen2019certified} were shown to be scalable to deep neural networks and high dimensional datasets enabling researchers to propose training procedures \cite{salman2019provably,zhai2020macer} that lead to models with high certified robustness. 
Thus, we assume that a victim will rely on RS based certification methods to measure the certified robustness and use RS based training procedures to train the models.
The fact that a victim can train with a method that improves certified adversarial robustness is an immediate challenge for current poisoning attacks which optimize the poison data to affect the accuracy of models trained with standard training. 
%(\JH{I changed the paragraph below.}) ok
Table~\ref{Table:difficulty_of_poisoning} shows that poisons optimized against standard training can significantly reduce the accuracy of the victim's model (left to right) when the victim also uses standard training (1st and 5th row). However, this poison data is rendered ineffective when the victim uses a certifiably robust training method such as \cite{cohen2019certified,salman2019provably,zhai2020macer}.
% on the poisoning data which has been optimized to reduce the accuracy of models trained with standard training. 
%As expected the effect of standard poisoning is significantly reduced if the victim trains using any randomized smoothing based certifiably robust training procedure.
%(\JH{We should ask here ``But how robust is RS if the data is poisoned against RS?})
%\BK{This is good idea -- basically say this might give a wrong or misleading impression that certified methods cannot be poisoned. In this paper we ask how robust RS really is if the data is poisoned?}

\emph{Are certified defenses robust to data poisoning?} 
We study this question and demonstrate that data poisoning is a serious concern even for certified defenses.  
%Motivated by the growing popularity of using certifiably robust training methods and the failure of current data poisoning attacks in affecting the standard accuracy of the robust models, we propose a novel data poisoning attack that degrades the certified robustness guarantees achievable using the robust training methods. 
%\BK{Instead of saying we propose, here we can answer the question asked in the previous sentence -- In this work, we show that even certified models can be made useless ..}
\begin{table}
  \caption{Failure of traditional data poisoning attacks optimized against standard training in affecting the test accuracy (of target class) of models trained with certifiably robust training procedures. Details of the experiment are present in Appendix~\ref{app:standard_acc}. Certifiably robust training methods \cite{cohen2019certified,salman2019provably,zhai2020macer} are trained with $\sigma$ = 0.25 and accuracy of their base classifiers are reported.
  %\JH{Show target accuracy only}
  %\AM{Should I show certified accuracy of smooth classifiers here?
  }
  \label{Table:difficulty_of_poisoning}
  \centering
  \small
  \resizebox{0.9\columnwidth}{!}{
    \begin{tabular}{c|c|c|c}
    \toprule
    &\multirow{1}{*}{\makecell{Training method }} & \multicolumn{1}{c|}{\makecell{Model trained \\ on clean data}} & \multicolumn{1}{c}{\makecell{Model trained \\ on  poison data}} \\
    \midrule
    \multirow{4}{*}{\STAB{\rotatebox[origin=c]{90}{MNIST}}} &
    Standard & 	99.28$\pm$0.01 & 60.08$\pm$12.6 \\
    &\makecell{GA\cite{cohen2019certified}} & 	98.99$\pm$0.14  & 98.31$\pm$1.65\\
    &SmoothAdv \cite{salman2019provably}&  99.18$\pm$0.23  & 99.31$\pm$0.29\\
    &MACER \cite{zhai2020macer}&  99.21$\pm$0.56  & 98.31$\pm$0.58\\
    \midrule
    \midrule
    \multirow{4}{*}{\STAB{\rotatebox[origin=c]{90}{CIFAR10}}} &
    Standard &  92.71$\pm$1.31  & 0.36$\pm$0.37 \\
    &\makecell{GA \cite{cohen2019certified}} &  88.84$\pm$2.39  & 88.38$\pm$2.13\\
    &SmoothAdv \cite{salman2019provably} & 79.48$\pm$2.69   &	74.95$\pm$3.45\\
    &MACER \cite{zhai2020macer}& 	87.12$\pm$1.17  &	88.54$\pm$4.52\\
    \bottomrule
    \end{tabular}
    }
\end{table}
We propose a novel data poisoning attack that can significantly compromise the certified robustness guarantees achievable from training with robust training procedures. We formulate the Poisoning Against Certified Defenses (PACD) attack as a constrained bilevel optimization problem and theoretically analyze its solution for the case when the victim uses linear classifiers. Our theoretical analysis and empirical results suggests that the decision boundary of the smoothed classifiers (used for RS) learned from 
%(\JH{changed ``learnt on'' to ``learned from'' everywhere}) ok
the poisoned data is significantly different from the one learned from clean data there by causing a reduction in certified radius. 
Our bilevel optimization based attack formulation is general since it can generate poisoned data against a model trained with any certifiably robust training method (lower-level problem) and certified with any certification procedure (upper-level problem). Fig.~\ref{fig:overview} shows the overview of the proposed PACD attack.

Unlike previous poisoning attacks that aim to reduce the accuracy of the models on a small subset of data, our attack can reduce the certified radius of an entire target class. 
The poison points generated by our attack have clean labels and imperceptible distortion making them difficult to detect. 
The poison data remains effective when the victim trains the models from scratch or uses data augmentation or weight regularization during training. Moreover, the attack points generated against a certified defense are transferable to models trained with other RS based certified defenses and to models with different architectures. 
This highlights the importance of training-data quality and curation for obtaining meaningful gains from certified defenses at test time, a factor not considered by current certified defense research.

%\AM{add a line to answer the question in the title}
Our main contributions are as follows
\begin{itemize}
    \item We study the problem of using data poisoning attacks to affect the robustness guarantees of classifiers trained using certified defense methods. To the best of our knowledge, this is the first clean label poisoning attack that significantly reduces the certified robustness guarantees of the models trained on the poisoned dataset. 
    
    \item We propose a bilevel optimization based attack which can generate poison data against several robust training and certification methods. We specifically use the attack to highlight the vulnerability of randomized smoothing based certified defenses to data poisoning.
    
    \item We demonstrate the effectiveness of our attack in reducing the certifiable robustness obtained using randomized smoothing on models trained with state-of-the-art certified defenses \cite{cohen2019certified,salman2019provably,zhai2020macer}. Our attack reduces the ACR of the target class by more than 30\%. 
    %\PY{(We should highlight some numbers here, like compared to standard poisoning, our attack reduces certified defense by xx\% and 00\% on MNIST and CIFAR, etc)}
\end{itemize}

\vspace{-0.3cm}
\section{Background and related work}
{\bf Randomized smoothing:}
%\PY{(need to mention RS is the sota certified defense first)}
The RS procedure \cite{cohen2019certified} uses a smoothed version of the original classifier $f:\mathbb{R}^d \xrightarrow{} \mathcal{Y}$ and certifies the adversarial robustness of the new classifier. The smoothed classifier, $g(x) = \arg \max_c \mathbb{P}_{\eta \sim \mathcal{N}(0, \sigma^2I)}(f(x+\eta)=c)$, assigns $x$ the class whose decision region $\{x' \in \mathbb{R}^d: f(x') = c\}$ has the largest measure under the distribution $\mathcal{N}(x, \sigma^2I)$, where $\sigma$ is used for smoothing. Suppose that while classifying a point $\mathcal{N}(x, \sigma^2I)$, the original classifier $f$ returns the class $c_A$ with probability $p_A = \mathbb{P}(f(x + \eta) = c_A)$, and the “runner-up” class $c_B$ is returned with probability $p_B = \max_{c \neq c_A} \mathbb{P}(f(x + \eta) = c)$, then the prediction of the point $x$ under the smoothed classifier $g$ is robust within the radius $r(g;\sigma) = \frac{\sigma}{2}(\Phi^{-1}(p_A) - \Phi^{-1}(p_B)),$ where $\Phi^{-1}$ is the inverse CDF of the standard Normal distribution. In practice, Monte Carlo sampling is used to estimate a lower bound on $p_A$ and an upper bound on $p_B$ as its difficult to estimate the actual values for $p_A$ and $p_B$.
Since standard training of the base classifier does not achieve high robustness guarantees, \cite{cohen2019certified} proposed to use GA based training in which the base classifier is trained on Gaussian noise corruptions of the clean data. Recent works \cite{zhai2020macer,salman2019provably} showed that the certified robustness guarantees of RS can be boosted by using different training procedures. In particular, \cite{salman2019provably} proposed to train the base classifier using adversarial training where the adversarial examples are generated against the smoothed classifier. 
%Their work showed a significant improvement in the certified robustness guarantees achievable with classifier trained with Gaussian augmented training based approach of \cite{cohen2019certified}.
Although effective at increasing the certified radius, the method can be slow to train due to the requirement of generating adversarial examples against the smoothed classifier at every step. Another recent work \cite{zhai2020macer} proposed a different training procedure which is significantly faster to train and relies on directly maximizing the certified radius for achieving high robustness guarantees. Due to their effectiveness in improving the certified robustness guarantees of machine learning models, we craft poison data against these methods. A recent attack method \cite{ghiasi2020breaking} showed that it is possible to fool a robust classifier to mislabel an input and give an incorrect certificate using perturbation large in $\ell_p$ norm at test-time. 
Our work is different since we focus on train-time attacks against certified defenses using imperceptibly small perturbations to the poison data.

{\bf Bilevel optimization:}
A bilevel optimization problem has the form $\min_{u \in \mathcal{U}} \xi(u,v^*)\;\mathrm{s.t.}\;v^* = \arg\min_{v\in \mathcal{V}(u)}\;\zeta(u,v)$, where the upper-level problem is a minimization problem with $v$ constrained to be the optimal solution to the lower-level problem (see \cite{bard2013practical}). Our data poisoning attack is a constrained bilevel optimization problem. Although general bilevel problems are difficult to solve, under some simplifying assumptions their solution can be obtained using gradient based methods. Several methods for solving bilevel problems in machine learning have been proposed previously \cite{domke2012generic,pedregosa2016hyperparameter,franceschi2017forward,maclaurin2015gradient,shaban2018truncated,mehra2019penalty} (See Appendix~\ref{app:approxgrad} for an overview). We use the method based on approximating the hypergradient by approximately solving a linear system (ApproxGrad Alg.~\ref{alg:approxgrad} in Appendix~\ref{app:approxgrad}) in this work. 
Previous works \cite{mei2015using,munoz2017towards,mehra2019penalty,huang2020metapoison,carnerero2020regularisation} have shown the effectiveness of solving bilevel optimization problem for data poisoning to affect the accuracy of models trained with standard training. Our work on the other hand proposes a bilevel optimization based formulation to generate a data poisoning attack against RS based certified defenses and shows its effectiveness against state-of-the-art robust training methods.
%effectiveness In this work we present a novel data poisoning attack formulation that can significantly compromise the certified robustness show that bilevel optimization based data poisoningformulated the data poisoning problem as a  and have shown its effectiveness at either reducing the overall accuracy of the poisoned model or reducing the accuracy on target points. Unlike their work, we focus on using data poisoning to affect the certified robustness guarantees of the models.

\section{Poisoning against certified defenses}
%\BK{Classifiers}
%\JH{I think this section is much more readable now, but details in 3.2 are still too detailed. Maybe it is unavoidable.}
Here we present the bilevel formulation of our PACD attack for generating poisoned data to compromise the certified robustness guarantees of the models trained using certified defenses. Specifically, we discuss how to generate poison data against GA \cite{cohen2019certified}, SmoothAdv \cite{salman2019provably} and MACER \cite{zhai2020macer} and affect the certified robustness guarantees obtained using RS.

\subsection{General attack formulation}
%Here we present our attack formulation to reduce the certified robustness guarantees of a general robust training procedure.
Let $\mathcal{D^\mathrm{clean}} = \{(x_i^\mathrm{clean}, y_i^\mathrm{clean})\}_{i=1}^{N_{\mathrm{clean}}}$ 
%(\JH{I changed the notation. This is the correct one to use}) ok
be the clean, unalterable portion of the training set. Let $u=\{u_1,...,u_n\}$ denote the attacker's poisoning data which is added to the clean data. 
%: $X^\mathrm{clean} \bigcup u$.
%\JH{$[X, u]$??}
For clean-label attack, we require that each poison example $u_i$ has a limited perturbation, for example, $\|u_i - x_i^\mathrm{base}\| =\|\delta_i\| \leq \epsilon$ from the base data $x_i^\mathrm{base}$ and has the same label $y_i^\mathrm{base}$, for $i=1,...,n$. Thus $\mathcal{D^\mathrm{poison}} = \{(u_i, y_i^\mathrm{base})\}_{i=1}^{N_{\mathrm{poison}}}$.
The goal of the attacker is to find $u$ such that when the victim uses $\mathcal{D^\mathrm{clean}} \bigcup \mathcal{D^\mathrm{poison}}$ to train a classifier, %using a robust training procedure,
the certified robustness guarantees of the model on the target class ($\mathcal{D^\mathrm{val}} = \{(x_i^\mathrm{val},y_i^\mathrm{val})\}_{i=1}^{N_\mathrm{val}}$) are significantly diminished compared to a classifier trained on clean data. 
The attack can be formulated as follows:
\begin{equation}
    \begin{split}
        &\min_{u \in \mathcal{U}}\;\; \mathcal{R}(\mathcal{D}^\mathrm{val}; \theta^\ast) \\
        %&\mathrm{s.t.}\;\; \|\delta_i\|_{\infty} \leq \epsilon,\;\;  i=1,...,n,\;\;\mathrm{and} \\ 
        \mathrm{s.t.}\;\; \theta^* = \arg&\min_{\theta}\; \mathcal{L}_\mathrm{robust}(\mathcal{D^\mathrm{clean}} \bigcup \mathcal{D^\mathrm{poison}}; \theta).
    \end{split}
    \label{eq:bilevel_simple}
\end{equation}
The upper-level cost $\mathcal{R}$ denotes a certified robustness metric such as the certified radius from RS. 
The goal of the upper-level problem is to compromise the certified robustness guarantees of the model trained on validation data $\mathcal{D}^\mathrm{val}$.
The solution to the lower-level problem $\theta^\ast$ are the parameters of the machine learning model learnt from $\mathcal{D^\mathrm{clean}} \bigcup \mathcal{D^\mathrm{poison}}$ using a robust training method with loss function $\mathcal{L}_\mathrm{robust}$. 
The fact that any training method that achieves high certified robustness to test-time attacks and any certification procedure can be incorporated into this formulation by changing the lower- and upper-level problems, respectively, makes the attack formulation broadly applicable. 
%\JH{Please use general in place as generic. The latter has a negative connotation}. 
Recent works \cite{cohen2019certified,salman2019provably,zhai2020macer} have shown RS based methods to be effective at certifying and producing robust classifiers.
The scalability of these methods to large datasets and deep models make them useful for real-world application. 
Thus, we focus on using our poisoning attack against these methods. 

\subsection{Poison randomized smoothing based defenses}
For an input at test time, RS produces a prediction from the smoothed classifier $g$ and a radius in which this prediction remains constant. 
Since the certified radius of a ``hard'' smooth classifier $g$ is non-differentiable, it cannot be directly incorporated in the upper-level of the attack formulation Eq.~(\ref{eq:bilevel_simple}). 
To overcome this challenge, we use the ``soft'' smooth classifier $\tilde{g}$ as an approximation. Similar technique has been used in \cite{salman2019provably, zhai2020macer}. 
Let $z_{\theta}:X\xrightarrow{}\mathcal{P}(K)$ be a classifier whose last layer is softmax with parameters $\theta$ and $\sigma > 0$ is the noise used for smoothing, then soft smoothed classifier $\Tilde{g}_{\theta}$ of $z_{\theta}$ is %defined as 
\(\Tilde{g}_{\theta}(x) = \arg \max_{c \in Y} \mathbb{E}_{\eta \sim \mathcal{N}(0,\sigma^2I)}[z^c_{\theta}(x + \eta)].\)
It was shown in \cite{zhai2020macer} that if the ground truth of an input $x$ is $y$ and $\Tilde{g}_{\theta}$ classifies $x$ correctly
%, i.e. $\mathbb{E}_{\eta}[z^y_{\theta}(x + \eta)] \geq \max_{y' \neq y}\mathbb{E}_{\eta}[z^{y'}_{\theta}(x + \eta)]$,
then $\tilde{g}_{\theta}$ is provably robust at $x$, with the certified radius $\tilde{r}(\tilde{g}_{\theta}; x, y, \sigma) = \frac{\sigma}{2}[\Phi^{-1}(\mathbb{E}_{\eta}[z^y_{\theta}(x + \eta)]) - \Phi^{-1}(\max_{y' \neq y}\mathbb{E}_{\eta}[z^{y'}_{\theta}(x + \eta)])].$
Assuming $\tilde{r}(\tilde{g}_{\theta}; x, y, \sigma) = 0$ when $x$ is misclassified, the ACR is $\tilde{R}(\tilde{g}_{\theta}; \mathcal{D}, \sigma) = \frac{1}{|\mathcal{D}|}\sum_{(x,y) \in \mathcal{D}} \tilde{r}(\tilde{g}_{\theta}; x, y, \sigma).$
%\JH{I changed $X$ to $D$ everywhere. If you want to make it mathcal please do so.} 
Since $\tilde{R}$ is differentiable we can use it in the upper-level of Eq.~(\ref{eq:bilevel_simple}). The lower-level problem can be any robust training procedure and we focus on using \cite{cohen2019certified,zhai2020macer,salman2019provably} in this work.

{\bf Poisoning against GA \cite{cohen2019certified}.} 
%\BK{Attacking maybe confusing..Compromising (or Poisoning) G D A based Certified Classifiers?} 
We start by showing how to generate poison data against models trained with GA \cite{cohen2019certified}, which was shown to yield higher certified robustness compared to models trained with standard training. In this method the classifier $f_{\theta}$ is obtained by optimizing the loss function $\mathcal{L}_\mathrm{GaussAug}(\mathcal{D};\theta,\sigma) = \frac{1}{|\mathcal{D}|}\sum_{(x_i,y_i) \in \mathcal{D}} l_{ce}(x_i+\eta,y_i;\theta)$, where $l_{ce}$ is the cross entropy loss and $\eta\sim \mathcal{N}(0,\sigma^2 I)$. To control the perturbation added to the poison data we used $\ell_\infty$-norm here but other norms can also be used. The bilevel formulation to generate poison data to reduce the certified robustness guarantees obtained using RS for a classifier trained with GA is as follows.
\begin{equation}
    \begin{split}
        &\min_{u}\;\; \tilde{R}(\tilde{g}_{\theta^*}; \mathcal{D}^\mathrm{val}, \sigma) \\
        &\mathrm{s.t.}\;\; \|\delta_i\|_{\infty} \leq \epsilon,\;\;  i=1,...,n,\;\;\mathrm{and} \\ 
        \theta^* = \arg&\min_{\theta}\; \mathcal{L}_\mathrm{GaussAug}(\mathcal{D^\mathrm{clean}} \bigcup \mathcal{D^\mathrm{poison}}; \theta,\sigma).
    \end{split}
    \label{Eq:Bilevel}
\end{equation}

\begin{figure*}[tb]
  \centering{\includegraphics[width=0.99\textwidth]{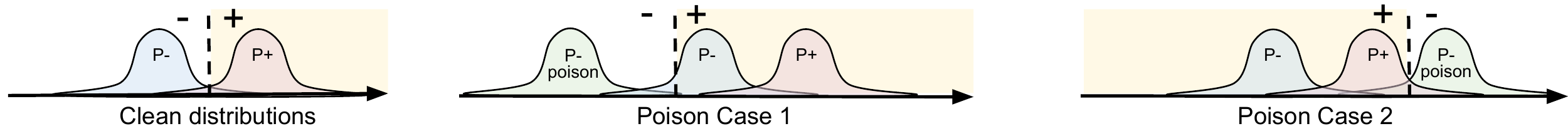}}
  \caption{Analytical solutions of problem (\ref{eq:bilevel_linear}) with linear classifiers. The poison distribution $(P^{-}_\mathrm{poison})$ can change the decision boundary (broken line) and reduce the ACR of the clean distribution $(P^{-})$ in two ways (Cases 1 and 2). Perturbation is exaggerated for illustration.}
  \label{fig:linear}
\end{figure*} 

\begin{algorithm}%[H] 
\caption{Poisoning GA based certified defense \cite{cohen2019certified}
%\PY{RS with Gaussian augmentation} 
} 
\label{alg:main}
\textbf{Input}: $\mathcal{D}^\mathrm{clean}, \mathcal{D}^\mathrm{base}, \mathcal{D}^\mathrm{val}, \mathrm{perturbation \;strength}\; \epsilon, \\ \mathrm{noise \;level}\; \sigma, \mathrm{number\; of\; noise\; samples}\; k, \\ \mathrm{inverse\; temperature}\;\alpha, \mathrm{total\; epochs}\; P,\\  \mathrm{lower-level\; epochs}\; T_1, \mathrm{epochs\; for \;linear\; system}\;T_2$ \\
\textbf{Output}: $\mathcal{D}^\mathrm{poison}$
%(\JH{Maybe itemize each line?})\\
%{Begin}
\begin{algorithmic}
\STATE{$\mathcal{D}^\mathrm{poison} := \mathcal{D}^\mathrm{base}$}
\FOR{$p=0,\;\cdots\;,P\textrm{-}1$}
    \STATE{Sample a mini-batch $(x^\mathrm{clean}, y^\mathrm{clean}) \;\sim\; \mathcal{D}^\mathrm{clean}$}
    \STATE{Sample a mini-batch of $n$ points $(x^\mathrm{val}, y^\mathrm{val}) \;\sim\; \mathcal{D}^\mathrm{val}$}
    \STATE{Sample a mini-batch $(x^\mathrm{poison}, y^\mathrm{poison}) \;\sim\;\mathcal{D}^\mathrm{poison}$}
    \STATE{Pick the base samples for poison data $(x^\mathrm{base}, y^\mathrm{base})$}
    \STATE{For each $x^\mathrm{val}_{i}$, sample $k$ i.i.d. Gaussian samples}
    \STATE{$x^\mathrm{val}_{i_1}, \cdots , x^\mathrm{val}_{i_n} \sim \mathcal{N}(x^\mathrm{val}_{i}, \sigma^2I)$}
    %\STATE{}
    \STATE{Compute $\tilde{z}_{\theta}(x^\mathrm{val}_i) \xleftarrow{} \frac{1}{k} \sum^k_{j=1} \alpha z_{\theta}(x^\mathrm{val}_{i_j})$ for each $i$}
    \STATE{$\mathcal{G_{\theta}} := \{(x^\mathrm{val}_i, y^\mathrm{val}_i): y^\mathrm{val}_i = \arg \max_{c \in \mathcal{Y}}\; \tilde{z}^c_{\theta}(x^\mathrm{val}_i)\}$}
    \STATE{}
    \STATE{For each $(x_i, y_i) \in \mathcal{G_{\theta}}$, compute $\tilde{y}_i$}
    \STATE{$\tilde{y}_i \xleftarrow{} \arg \max_{c \in \mathcal{Y}\textbackslash\{y_i\}} \tilde{z}^c_{\theta}(x_i)$}
    \STATE{For each  $(x_i, y_i) \in \mathcal{G_{\theta}}$, compute $\tilde{r}(x_i, y_i)$}
    \STATE{$\tilde{r}(x_i, y_i) = \frac{\sigma}{2}(\Phi^{-1}(\tilde{z}^{y_i}_{\theta}(x_i)) - \Phi^{-1}(\tilde{z}^{\tilde{y}_i}_{\theta}(x_i)))$}
    \STATE{}
    \STATE{$\xi :=  \frac{1}{n}\sum_{(x_i, y_i) \in \mathcal{G_{\theta}}}\tilde{r}(x_i, y_i)$}
    \STATE{$\zeta :=  \mathcal{L}_\mathrm{GaussAug}((x^\mathrm{clean}, y^\mathrm{clean}) \bigcup (x^\mathrm{poison}, y^\mathrm{poison}), \sigma)$}
    \STATE{$(x^\mathrm{poison}, y^\mathrm{poison}):=$ApproxGrad$(\xi, \zeta, 1, T_1, T_2,\epsilon,x^{base})$}
    \STATE{Update $\mathcal{D}^\mathrm{poison}$ with $(x^\mathrm{poison}, y^\mathrm{poison})$}
\ENDFOR
\end{algorithmic}
\end{algorithm}

{\bf Poisoning against MACER \cite{zhai2020macer}.} Another recent work proposed a method for robust training by maximizing the certified radius (MACER). Their approach uses a loss function which is a combination of the classification loss and the robustness loss of the soft smoothed classifier $\tilde{g}_\theta$. In particular, the loss of the smoothed classifier on a point $(x,y)$ is given by $l_{macer}(\tilde{g}_\theta; x, y) = -\mathrm{log} \;\hat{z}^y_{\theta}(x) + \frac{\lambda \sigma}{2} \mathrm{max}\{\gamma - \tilde{\xi}_{\theta}(x, y), 0\}\cdot\mathbf{1}_{\tilde{g}_{\{\theta}(x)=y\}}$. where $\eta_1, ..., \eta_k$ are $k$ i.i.d. samples from $\mathcal{N}(0, \sigma^2\mathbf{I})$, $\hat{z}^y_{\theta}(x) = \frac{1}{k}\Sigma^k_{j=1} z^y_{\theta}(x + \eta_j)$ is the empirical expectation of $z_{\theta}(x + \eta)$, $\tilde{\xi}_{\theta}(x, y) = \Phi^{-1}(\hat{z}^y_{\theta}(x)) - \Phi^{-1}(\mathrm{max}_{y \neq y} \hat{z}^{y'}_{\theta}(x))$, $\gamma$ is the hinge factor, and $\lambda$ balances the accuracy and robustness trade-off. Using this we can define $\mathcal{L}_\mathrm{macer}(\mathcal{D};\theta,\sigma) = \frac{1}{|\mathcal{D}|}\sum_{(x_i,y_i) \in \mathcal{D}} l_{macer}(\tilde{g}_{\theta};x_i,y_i)$. To generate poison data that reduces the robustness guarantees of classifier trained with MACER we can use the loss $\mathcal{L}_\mathrm{macer}(\mathcal{D};\theta,\sigma)$ in the lower-level problem in Eq.~(\ref{Eq:Bilevel}).

{\bf Poisoning against SmoothAdv \cite{salman2019provably}.}  It was shown that the certified robustness guarantees obtained from RS can be improved by training the classifiers using adversarial training with adversarial examples generated against the smooth classifier. In particular the classifier trained with SmoothAdv optimizes the following objective for a point $(x,y)$.
$\min_{\theta}\max_{\|x' - x\|_2 \leq \alpha} -\mathrm{log} \frac{1}{k}\Sigma^k_{j=1} z^y_{\theta}(x' + \eta_j)$ where $\eta_1, ..., \eta_k$ are $k$ i.i.d. samples from $\mathcal{N}(0, \sigma^2\mathbf{I})$ and $\alpha$ is the permissible $\ell_2$ distortion to $x$. To generate poisoning data against SmoothAdv we must use this objective as the lower-level problem in Eq.~(\ref{Eq:Bilevel}). To make it easier for bilevel solvers to solve this problem we use an approximation to the mini-max problem. For doing that we first compute the adversarial example 
$x' = \arg\max_{\|x' - x\|_2 \leq \alpha} -\mathrm{log} \frac{1}{k}\Sigma^k_{j=1} z^y_{\theta}(x' + \eta_j)$ using PGD attack on the points in $\mathcal{D^\mathrm{clean}} \bigcup \mathcal{D^\mathrm{poison}}$ and then use these examples as our new dataset to train the model parameters in the lower-level as in Eq.~(\ref{Eq:Bilevel}). Specifically, the lower-level problem in Eq.~(\ref{Eq:Bilevel}) becomes $\arg\min_{\theta}\; \mathcal{L}_\mathrm{GaussAug}(\mathcal{D}^\mathrm{clean}_{adv} \bigcup \mathcal{D}^\mathrm{poison}_{adv}; \theta,\sigma)$ where $\mathcal{D}_{adv}$ denotes the adversarial examples generated against $\tilde{g}_{\theta}$. We update $\mathcal{D}_{adv}$ in every step of the bilevel optimization.
\vspace{-0.1cm}
\subsection{Generation and evaluation of poisoning attack}
In this work, we focus on creating a poisoned set to compromise the certified adversarial robustness guarantees of all points in a target class. We initialize the poison data with clean data from the target class (i.e., base data) and optimize the perturbation to be added to each point by solving the bilevel problem in Eq.~(\ref{Eq:Bilevel}) for attack against GA based training. We use a small value of $\epsilon$ to ensure the perturbations added are imperceptible and the poison points have clean labels when inspected visually (See Fig.~\ref{fig:attack_egs} in the Appendix). The bilevel optimization is solved using the ApproxGrad algorithm (Alg.~\ref{alg:approxgrad} in Appendix~\ref{app:approxgrad}). The full attack algorithm for generating poison data against GA \cite{cohen2019certified} is shown in Alg.~\ref{alg:main}. Attack against other methods are generated similarly by replacing the lower-level objective ($\zeta$ in Alg.~\ref{alg:main}) with the appropriate loss function for MACER~\cite{zhai2020macer} and SmoothAdv~\cite{salman2019provably}. We evaluate the effect of poisoning, by training the models from scratch using GA, MACER and SmoothAdv on their respective poisoned sets and report ACR and approximate certified accuracy (points with certified $\ell_2$ radius greater than zero) on the clean test points from the target class. Previous works \cite{huang2020metapoison, shafahi2018poison} had shown the effectiveness of poisoning by lowering the accuracy on specific target points from the test set. Our attack is also effective, under a similar setting, at reducing the certified radius for target points (Appendix~\ref{app:partial_poisoning}).

\vspace{-0.1cm}
\subsection{Analysis of poisoning with linear classifiers}
%\AM{New}
%\BK{Motivate this scetion: why this analysis is useful..e.g., To gain a deeper insight into the solution, we ...}
To gain a deeper insight into the effect of poisoning, we analyze the %optimal poisoning for reducing the average certified radius and compare it to the 
analytical solution of our bilevel problem for the case of linear classifiers trained with GA. %\AM{GA or standard training}. 
%We solve the bilevel optimization (\ref{eq:bilevel_linear}) analytically for linear classification in one dimension and compare with numerical solution using Alg~\ref{alg:main}. 
%({\bf JH} Is this true? Or are we using logistic regression?) 
Suppose we have a one-dimensional two-class problem and the attacker's goal is to poison the distribution of the \emph{negative} class $P^{-}$ so that the ACR ($\tilde{R}$) of the poisoned model on the test points of the \emph{negative} class is reduced. 
Let $\epsilon$ be the the maximum permissible perturbation that can be added by the attacker to the points of the class $P^{-}$. 
% be bounded by $|u_i-x_i^{-}| \leq \epsilon,\;\mathrm{for}\;i=1,...,n$. 
We do not assume any specific distributions for $P^{+}$ and $P^{-}$ here, but only that $\sum_i x_i^{-}<\sum_i x_i^{+}$ without loss of generality. Here $x_i^{+}$ and $x_i^{-}$ refer to the training points of the positive and the negative class, respectively.
A linear classifier in one-dimension is either $f(x)=1\;\mathrm{iff}\;x \geq t$ or $f(x)=-1\;\mathrm{iff}\;x \leq t$ parameterized by the threshold $t$. 
%({\bf JH} $\theta$ was used with a different meaning in (1)). 
For linear classifiers, 
%it is known \cite{cohen2019certified} that for any value of $\sigma$ used for smoothing, 
the smoothed classifier $g$ is the same as the unsmoothed classifier $f$ and the certified radius for a point is the distance to the decision boundary \cite{cohen2019certified}. % $CR(x,y) = \max(0, y(x+b/w)) $. 
%To make the problem analytically tractable,  we use only the radius term without the accuracy term in the upper-level cost. 
To make the problem analytically tractable, we use the squared-loss at the lower-level i.e., $f(x) = wx + b$ and $l(x,y;f)=(f(x) - y)^2$.
%({\bf JH} If more space is needed, the formula below can be moved to Appendix)
The bilevel problem for poisoning is as follows
\begin{equation}
    \begin{split}  
        &\min_{u}\;\; \mathbb{E}_{{P}_{-}}[\max(\mathrm{sign}(w^\ast)(-b^\ast/w^\ast-x),0)]\\
        &\mathrm{s.t.}\;\; -\epsilon \leq u_i - x_i^{-} \leq \epsilon, \;\; \mathrm{for \;\; i = 1, ..., n}\\ 
        %w^\ast,b^\ast = &\arg\min_{w, b}\; \frac{1}{2n}\big[\sum_{i=1}^{n}(wx^{+}_i + b - 1)^2 + \sum_{i=1}^{n}(wu_i + b + 1)^2\big].
        w^\ast,&b^\ast = \arg\min_{w, b}\; \frac{1}{2n}\big[\sum_{i=1}^{n}l(x_i^{+},1) + \sum_{i=1}^{n}l(u_i,-1)\big].
    \end{split}
    \label{eq:bilevel_linear}
\end{equation}
 %Since the solution to the lower-level problem can be computed in closed form we can replace the lower-level problem with its optimal solution. %$w^* = \frac{-2n}{\alpha}\sum_{i=1}^{n} (u^i - x^i_{+}), b^*=\frac{1}{\alpha}\big[(\sum_{i=1}^{n}(u^i - x^i_{+}))(\sum_{i=1}^{n}(u^i + x^i_{+}))\big]$ where $\alpha = \sum_{i=1}^n \big[2n((u^i)^2 + (x^i_{+})^2) - (u^i + x^i_{+})^2\big]$. 
\begin{theorem}
\label{thm:linear}
If the perturbation is large enough, i.e., $\epsilon \geq \frac{\sum_i x_i^{+} - \sum_i x_i^{-}}{n}$ then there are two locally optimal solutions to (\ref{eq:bilevel_linear}) which are $u_i = x_i^{-} - \epsilon$ (Case 1) and $u_i = x_i^{-} + \epsilon$ (Case 2) for $i=1,...,n$. 
Otherwise, there is a unique globally optimal solution $u_i = x_i^{-} - \epsilon$ (Case 1) for $i=1,...,n$.
\end{theorem}
%{\bf For large $\epsilon$?}
Thus, optimal poisoning is achieved by shifting points of the $P^{-}$ class either towards left or right by the maximum amount $\epsilon$ (Fig.~\ref{fig:linear} and Appendix~\ref{app:isotropic_gaussian}). Moreover, the effect of poisoning an $\alpha$ fraction of points from the $P^{-}$ class with maximum permissible perturbation $\Tilde{\epsilon}$ is same as that of poisoning all points of $P^{-}$ class with $\epsilon = \alpha \Tilde{\epsilon}$ (Corollary~\ref{cor:partial_poisoning} in Appendix~\ref{app:proofs}).
%In the linear case, reduction in the radius due to the change in the decision boundary also incurs the loss of accuracy in the target class (more so in Case 2 than Case 1). 
Although a direct analysis is intractable for non-linear cases, we empirically observed that our attack moved the decision boundary of neural networks closer to the points of the target class as measured by the mean distance of points to the decision boundary of the smoothed classifier (Sec.~\ref{sec:exp_emp_robustness}).
%Furthermore, for nonlinear classifiers, it is feasible to reduce the radius without degrading the accuracy too much (See Tables~\ref{Table:mnist} and~\ref{Table:cifar10} in Sec.~\ref{sec:experiments}).
%r and find that this distance reduces after poisoning suggesting that poisoning moves the decision boundary closer to the test points of the target class. in 
%\BK{We do not analyze the stealthy attack right? It maybe useful to point out that for DNNs one can minimize CR without any noticeable effect on accuracy -- This is different than what we in the Thm}

\vspace{-0.1cm}
%%%%%%%%%%%%%%%%%%%%%%%%%%%%%%%%%%%%%%%%%%%%%%%%%%
\section{Experiments}\label{sec:experiments}
%%%%%%%%%%%%%%%%%%%%%%%%%%%%%%%%%%%%%%%%%%%%%%%%%%
%({\bf JH} Suggested baseline attacks:  Move all digits of 8 to 1) random independent directions, 2) along the direction from the center of digit 8 to the center of 3, and 3) the opposite direction. The constraints $\epsilon$ should be applied in all cases.)
In this section we present the results of our PACD\footnote{The code is available at \url{https://github.com/akshaymehra24/poisoning_certified_defenses}} attack on poisoning deep neural networks trained using methods that make the model certifiably robust to test-time attacks. All the results presented here are averaged over models trained with five random initialization. 
%For the experiments with CIFAR10 with \cite{cohen2019certified} and \cite{zhai2020macer} we used data augmentation using random cropping and random flipping during attack generation and evaluation.
We report the average certified radius (ACR) as the average of the certified radius obtained from the RS based certification procedure of \cite{cohen2019certified} for correctly classified points. Certified radius is zero for misclassified and abstained points. 
The approximate certified accuracy (ACA) is the fraction of points correctly classified by the smoothed classifier ($\ell_2$ radius greater than zero). All results are reported over 500 randomly sampled images from the target classes.
We use the same value of $\sigma$ for smoothing during attack, retraining and evaluation. We compare our results to watermarking \cite{shafahi2018poison} 
%\BK{should we mention the objective of waterm is to min acc? also the intention to justify fairness} 
which has been used previously for clean label attacks (opacity 0.1 followed by clipping to make $\ell_{\infty}$ distortion equal to $\epsilon$),
and show that poison data generated using the bilevel optimization is significantly better at reducing the average certified radius. 

\begin{table}[tb]
  \caption{Decrease in certified radius and certified accuracy of models trained with Gaussian augmentation \cite{cohen2019certified} on poison data compared to those of models trained on clean and watermarked data.
  %Comparison of certified adversarial robustness of the classifier trained on poisoning data using Gaussian data augmentation with digit 8 from MNIST and class ``ship'' from CIFAR10 being the target class. 
  }
  \label{Table:cohen_attack}
  \centering
  \small
  \resizebox{0.95\columnwidth}{!}{
    \begin{tabular}{c|c|c|cc}
    \toprule
    & \multirow{2}{*}{$\sigma$} & \multirow{2}{*}{Data} &  \multicolumn{2}{c}{\makecell{Certified robustness of target class}}\\
    & & & ACR & ACA(\%) \\
    %\midrule
    %\multicolumn{4}{l}{\makecell{MNIST}} \\
    \midrule
    \multirow{9}{*}{\STAB{\rotatebox[origin=c]{90}{MNIST}}} &
    \multirow{3}{*}{0.25} & Clean   & 0.896$\pm$0.01 &	98.92$\pm$0.32 \\
    & & Watermarked & 0.908$\pm$0.01 & 99.24$\pm$0.29 \\
    & & Poisoned & {\bf0.325}$\pm$0.10 &	{\bf71.96}$\pm$8.28 \\
    \cmidrule{2-5}
    &\multirow{3}{*}{0.5} & Clean & 1.481$\pm$0.02 &	99.16$\pm$0.34\\
    &  & Watermarked & 1.514$\pm$0.06 &	99.12$\pm$0.47\\
    & & Poisoned  & {\bf0.733}$\pm$0.10 &	{\bf90.68}$\pm$3.37 \\
    \cmidrule{2-5}
    &\multirow{3}{*}{0.75} & Clean & 1.549$\pm$0.11 & 98.48$\pm$0.35\\
    & & Watermarked & 1.566$\pm$0.06 &	98.36$\pm$0.39\\
    &  & Poisoned  & {\bf0.698}$\pm$0.13 &	{\bf84.92}$\pm$5.14 \\
    \midrule
    \midrule
    \multirow{6}{*}{\STAB{\rotatebox[origin=c]{90}{CIFAR10}}} &
    \multirow{3}{*}{0.25} & Clean   & 0.521$\pm$0.05 &	85.76$\pm$3.31 \\
    & & Watermarked & 0.470$\pm$0.01 &	83.22$\pm$1.41 \\
    & & Poisoned & {\bf0.059}$\pm$0.02 &	{\bf26.84}$\pm$6.04 \\
    \cmidrule{2-5}
    & \multirow{3}{*}{0.5} & Clean & 0.634$\pm$0.04 &	75.04$\pm$1.65\\
    & & Watermarked & 0.611$\pm$0.18 &	74.01$\pm$9.22\\
    & & Poisoned  & {\bf0.221}$\pm$0.04 &	{\bf42.28}$\pm$6.01 \\
    
    \bottomrule
    \end{tabular}
    }
\end{table}

We use our attack to poison MNIST and CIFAR10 dataset and use ApproxGrad to solve the bilevel optimization. The time complexity for ApproxGrad is $O(VT)$ where $V$ are the number of parameters in the machine learning model and $T$ is the number of lower-level updates. For datasets like Imagenet where the optimization must be performed over a very large number of batches, obtaining the solution to bilevel problems becomes computationally hard. Due to this bottleneck we leave the problem of poisoning Imagenet for future work. For the experiments with MNIST we randomly selected digit 8 and for CIFAR10 the class ``Ship'' as the target class for the attacker. 
%(\JH{Don't we have results from all target classes yet? Are you planning to? Need to say something about it.}) 
The attack results for other target classes are similar and are presented in the Appendix~\ref{app:additional_experiments}.
To ensure that the attack points satisfy the clean label constraint, the maximum permissible $\ell_{\infty}$ distortion is bounded by $\epsilon = 0.1$ for MNIST and $\epsilon = 0.03$ for CIFAR10 which is similar to the value used to generate imperceptible adversarial examples in previous works \cite{madry2017towards,goodfellow2014explaining}. We used  convolutional neural networks for our experiments on MNIST and Resnet-20 model for our experiments with CIFAR10. Model architectures, hyperparameters, generated attack examples (Fig.~\ref{fig:attack_egs} in Appendix), and additional results on transferability of our poisoned samples to models with different architectures 
%\BK{to different architectures}
are presented in Appendix \ref{app:additional_experiments}. 
Since models trained with standard training do not achieve high certified radius \cite{cohen2019certified}, we considered poisoning models trained with methods that improve the certified robustness guarantees to test time attacks.
%In this work we present poisoning attack with models trained using robust training procedures since standard training does not achieve high certified robustness with randomized smoothing, due to it's use of a smooth classifier for certification.
%(\JH{What does ``focused on'' means?}) 
For comparison, ACR on the target class ``Ship'' with Resnet-20 trained with standard training on clean CIFAR10 dataset is close to zero whereas for the same model trained with GA ($\sigma = 0.25$) ACR is close to 0.5. 
Finally, we show that our attack can withstand the use of weight regularization, which has been shown to be effective at mitigating the effect of poisoning attacks \cite{carnerero2020regularisation}. The results of this experiment are present in Appendix~\ref{app:weight_reg}.
%A popular method to obtain empirical robustness is to use adversarial training. However, due to the lack of certified robustness guarantee of this method we do not consider it here. 
%\BK{This may be confusing as we have smooth adv}

\begin{table}[tb]
  \caption{Decrease in certified radius and certified accuracy of models trained with MACER \cite{zhai2020macer} on poison data compared to those of models trained on clean and watermarked data.
  }
  \label{Table:macer_attack}
  \centering
  \small
  \resizebox{0.95\columnwidth}{!}{
    \begin{tabular}{c|c|c|cc}
    \toprule
    &\multirow{2}{*}{$\sigma$} & \multirow{2}{*}{Data} & 
    \multicolumn{2}{c}{\makecell{Certified robustness of target class}}\\
    && 
    & ACR & ACA(\%) \\
    \midrule
    \multirow{9}{*}{\STAB{\rotatebox[origin=c]{90}{MNIST}}} &
    \multirow{3}{*}{0.25} & Clean   & 0.915$\pm$0.01	& 99.64$\pm$0.21\\
    & & Watermarked & 0.894$\pm$0.01 &	98.84$\pm$0.53\\
    & & Poisoned & {\bf0.431}$\pm$0.13 &	{\bf79.81}$\pm$9.26 \\
    \cmidrule{2-5}
    &\multirow{3}{*}{0.5} & Clean & 1.484$\pm$0.11 &	98.56$\pm$0.41\\
    & & Watermarked & 1.475$\pm$0.08 &	98.68$\pm$0.39\\
    & & Poisoned  & {\bf0.685}$\pm$0.16 &	{\bf84.36}$\pm$6.17 \\
    \cmidrule{2-5}
    &\multirow{3}{*}{0.75} & Clean & 1.353$\pm$0.13 &	93.81$\pm$2.08\\
    & & Watermarked & 1.415$\pm$0.11 &	94.52$\pm$1.58\\
    & & Poisoned  & {\bf1.008}$\pm$0.19 &	{\bf88.41}$\pm$4.64 \\
     
     \midrule
     \midrule
     \multirow{6}{*}{\STAB{\rotatebox[origin=c]{90}{CIFAR10}}} &
     \multirow{3}{*}{0.25} & Clean   & 0.593$\pm$0.05 &	83.84$\pm$2.26 \\
    & & Watermarked & 0.486$\pm$0.04 &	77.01$\pm$0.21 \\
    & & Poisoned & {\bf0.379}$\pm$0.11 &	{\bf72.41}$\pm$9.79\\
    \cmidrule{2-5}
    & \multirow{3}{*}{0.5} & Clean & 0.759$\pm$0.11 &	72.92$\pm$5.06\\
    & & Watermarked & 0.811$\pm$0.10	& 75.66$\pm$2.99\\
    & & Poisoned  & {\bf0.521}$\pm$0.11 &	{\bf65.24}$\pm$6.55 \\
    \bottomrule
    \end{tabular}
    }
\end{table}

\vspace{-0.1cm}
\subsection{Poisoning Gaussian data augmentation \cite{cohen2019certified}}
%\BK{Maybe more informative title?: How robust is GDA based Certified Defense?}
Here we show the effectiveness of our attack at compromise the certified robustness guarantees obtained with RS on a model trained using the GA. The results of the attack, present in Table~\ref{Table:cohen_attack}, show a significant decrease in the ACR and the certified accuracy of the target class of the model trained on poisoned data compared to the model trained on clean and watermarked data. Since the certified radius and certified accuracy are correlated, our poisoning attack which targets the reduction of certified radius (upper-level problem in Eq.~(\ref{Eq:Bilevel})) also causes a decrease in the certified accuracy. Significant degradation in average certified radius from 0.52 to 0.06 on CIFAR10 with imperceptibly distorted poison data shows the extreme vulnerability of GA to poisoning.

\vspace{-0.1cm}
\subsection{Poisoning MACER \cite{zhai2020macer}}
Here we use the bilevel formulation in Eq.~(\ref{Eq:Bilevel}) with $\mathcal{L}_\mathrm{macer}$ loss in the lower-level and generate poison data to reduce the certification guarantees of models trained with MACER. The poison data is generated with $k=2$, where $k$ are the number of noisy images used per training point to ease the bilevel optimization. However, during retraining $k=16$ is used, which is similar to the one used in the original work \cite{zhai2020macer}. 
The ACR obtained using MACER is higher than that achievable using Gaussian augmentation based training consistent with \cite{zhai2020macer}. However, our data poisoning attack is still able to reduce the average certified radius of the method by more than 30\% (Table~\ref{Table:macer_attack}) even though the attack is evaluated against a much stronger defense ($k=16$ for retraining compared to $k=2$ for poisoning) than what the poison data was optimized against. This shows that the use of a larger number of noisy samples ($k$) cannot eliminate the effect of the attack, emphasising the importance of the threat posed by data poisoning.
%The results though we generated poison data against a weak form of MACER, results in Table~\ref{Table:macer_attack} suggests that the attack is still very effective in reducing the average certified radius. 

\begin{table}[tb]
  \caption{%\BK{More informative:}
  Decrease in certified radius and certified accuracy of models trained with SmoothAdv \cite{salman2019provably} on poison data compared to those of models trained on clean and watermarked data.
  %Comparison of certified adversarial robustness of the classifier trained on poisoning data using SmoothAdv \cite{salman2019provably} with digit 8 from MNIST and class ``ship'' from CIFAR10 being the target class.
  }
  \label{Table:smoothadv_attack}
  \centering
  \small
  \resizebox{0.95\columnwidth}{!}{
    \begin{tabular}{c|c|c|cc}
    \toprule
    &\multirow{2}{*}{$\sigma$} & \multirow{2}{*}{Data} & 
    \multicolumn{2}{c}{\makecell{Certified robustness of target class}}\\
    && 
    & ACR & ACA(\%) \\
    \midrule
    \multirow{9}{*}{\STAB{\rotatebox[origin=c]{90}{MNIST}}} &
    \multirow{3}{*}{0.25} & Clean   & 0.896$\pm$0.01 &	99.16$\pm$0.45 \\
    & & Watermarked & 0.906$\pm$0.01 &	99.28$\pm$0.16 \\
    & & Poisoned & {\bf0.672}$\pm$0.04 &	{\bf93.21}$\pm$1.92 \\
    \cmidrule{2-5}
   & \multirow{3}{*}{0.5} & Clean & 1.408$\pm$0.05 &	99.21$\pm$0.25\\
    & & Watermarked & 1.401$\pm$0.02 &	98.01$\pm$0.18\\
    & & Poisoned  & {\bf1.037}$\pm$0.06 &	{\bf93.81}$\pm$1.31 \\
   \cmidrule{2-5}
   & \multirow{3}{*}{0.75} & Clean & 1.262$\pm$0.05 &	95.68$\pm$0.47\\
    & & Watermarked & 1.433$\pm$0.03 &	97.21$\pm$0.13\\
    & & Poisoned  & {\bf0.924}$\pm$0.06 &	{\bf88.88}$\pm$0.18 \\
    \midrule
     \midrule
     \multirow{6}{*}{\STAB{\rotatebox[origin=c]{90}{CIFAR10}}} &
    \multirow{3}{*}{0.25} & Clean   & 0.504$\pm$0.02	& 78.76$\pm$0.81 \\
    & & Watermarked & 0.441$\pm$0.02 & 70.16$\pm$2.12 \\
    & & Poisoned & {\bf0.271}$\pm$0.02 &	{\bf55.78}$\pm$0.96 \\
    \cmidrule{2-5}
    &\multirow{3}{*}{0.5} & Clean & 0.479$\pm$0.07 &	65.84$\pm$4.81\\
    & & Watermarked & 0.473$\pm$0.02 &	62.51$\pm$2.12\\
   &  & Poisoned  & {\bf0.277}$\pm$0.02 &	{\bf49.11}$\pm$3.19 \\
    
    \bottomrule
    \end{tabular}
    }
\end{table}

\begin{figure}[tb]
  \centering
  \includegraphics[width=0.6\columnwidth]{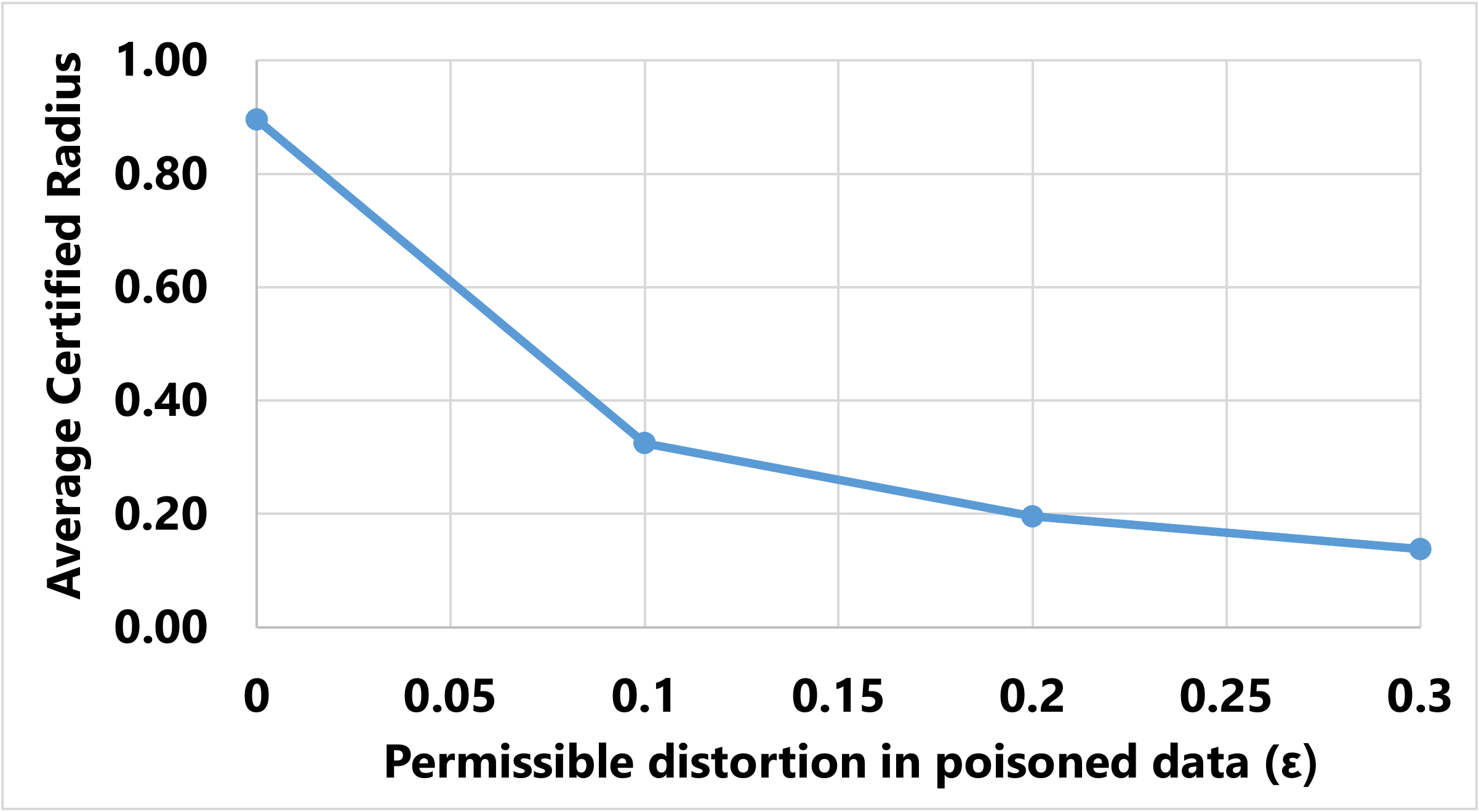}
  \includegraphics[width=0.6\columnwidth]{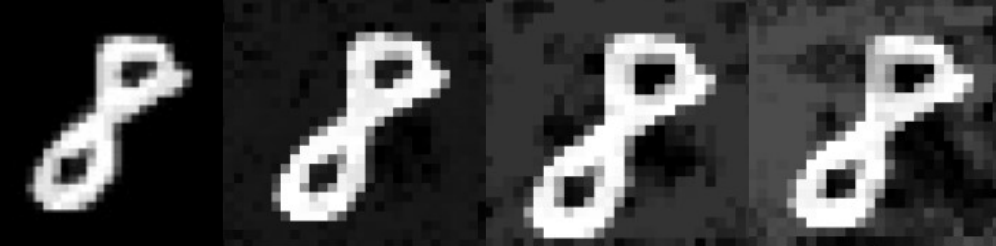}
  \caption{%\BK{Do we have similar results for CIFAR?} no
  (Upper) ACR degrades more if larger perturbations are permitted to create poison data. But larger perturbation makes the poison points visibly distorted making them easier to detect with inspection (Lower).
  %Effect of using different $\epsilon$ values for the permissible distortion of poison data. Larger $\epsilon$ causes a larger reduction in the average certified radius of the models but the attack points have have large distortion making the attack detectable. 
  Poison data are generated with $\epsilon \in \{0,0.1,0.2,0.3\}$. We have used $\epsilon = 0.1$ for our attacks.}
  \label{fig:acr_vs_epsilon}
\end{figure}

\begin{figure*}[tb]
  \centering{
  \subfigure[Average certified radius of digit 8 in MNIST]{\includegraphics[width=0.775\columnwidth]{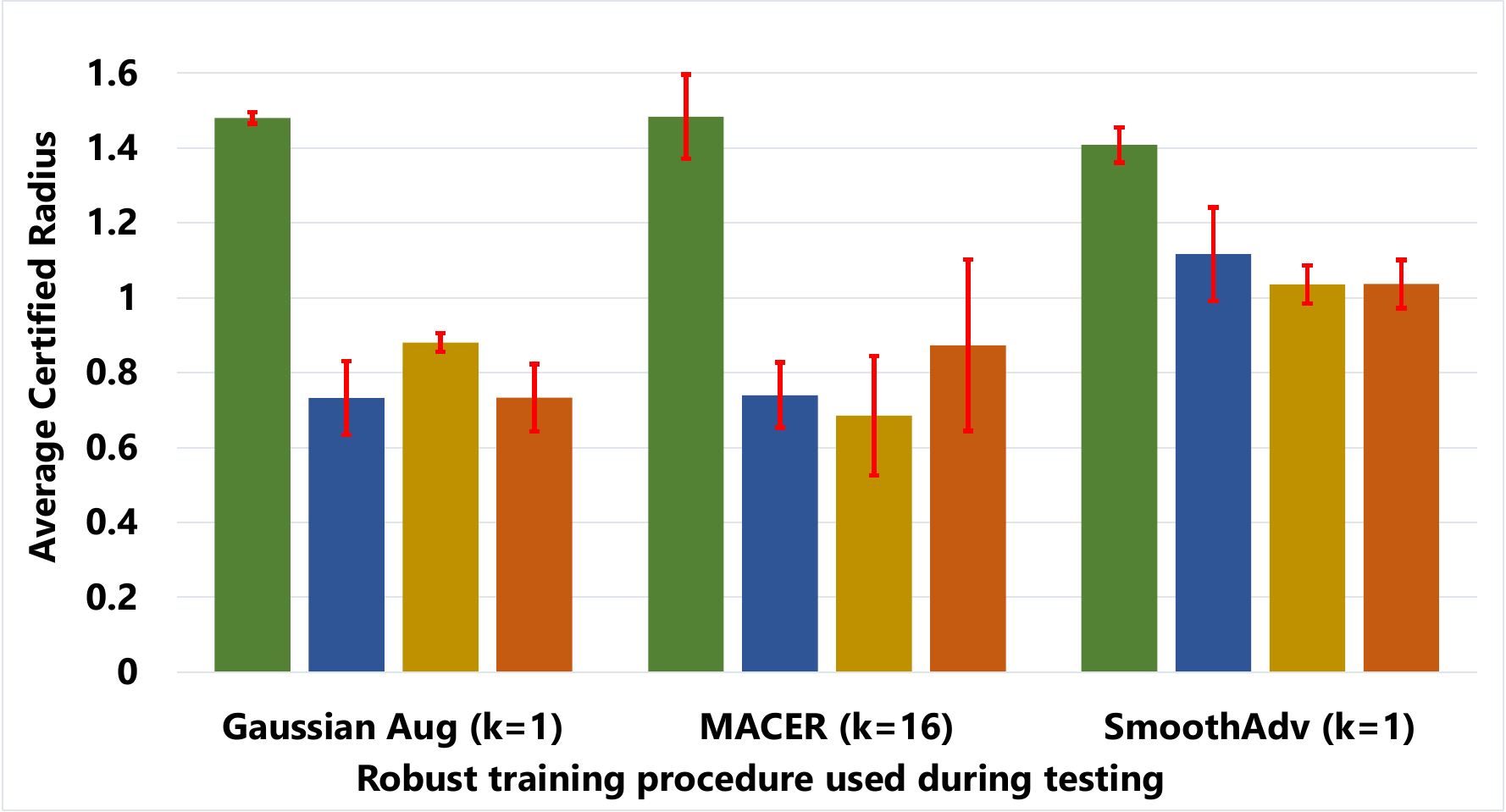}}
  \hspace{0.15in}
  \subfigure[Approximate certified accuracy of digit 8 in MNIST]{\includegraphics[width=0.775\columnwidth]{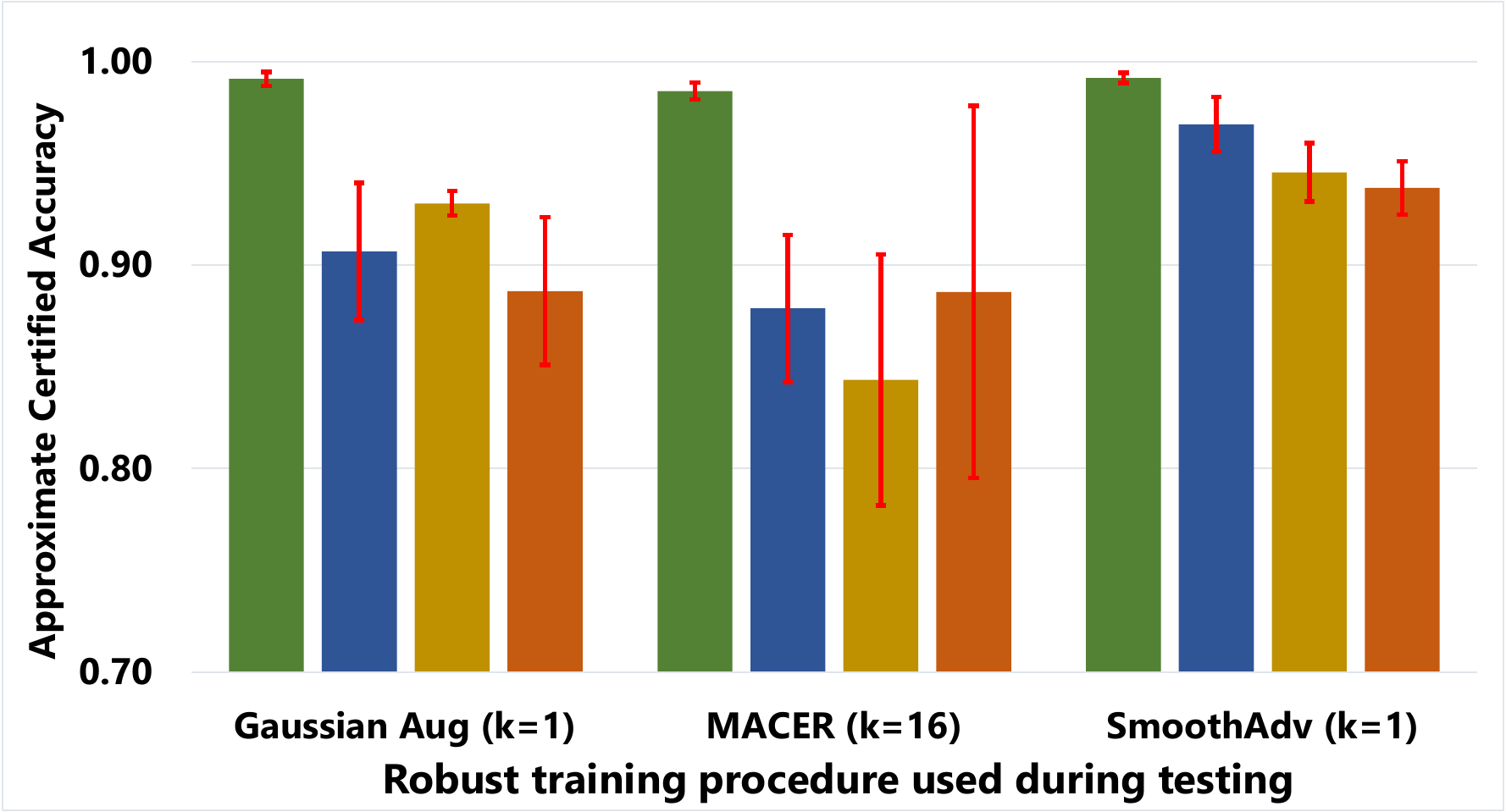}}
  \subfigure[Average certified radius of ``Ship'' in CIFAR10]{\includegraphics[width=0.775\columnwidth]{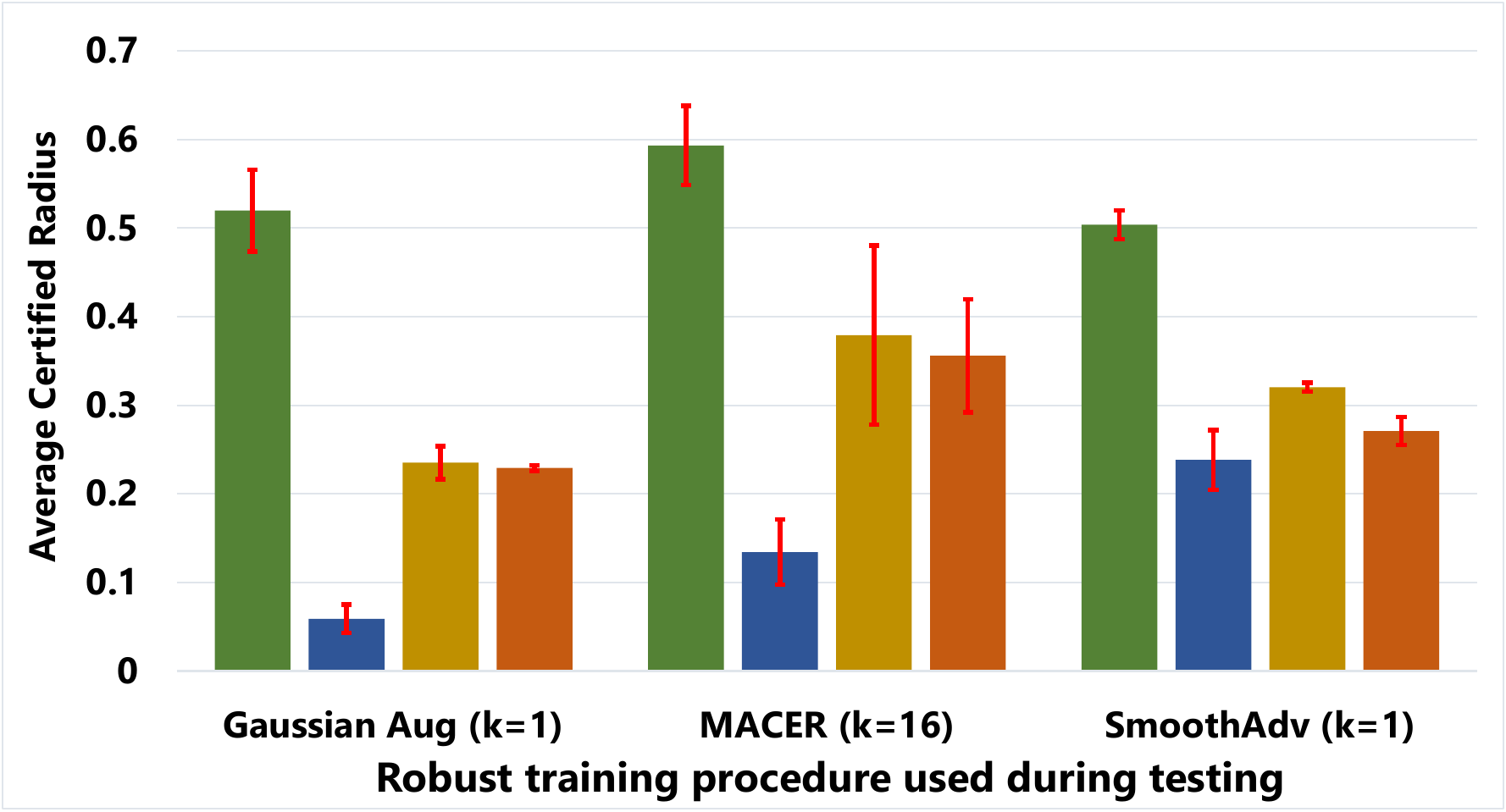}}
  \hspace{0.15in}
  \subfigure[Approximate certified accuracy of ``Ship'' in CIFAR10]{\includegraphics[width=0.775\columnwidth]{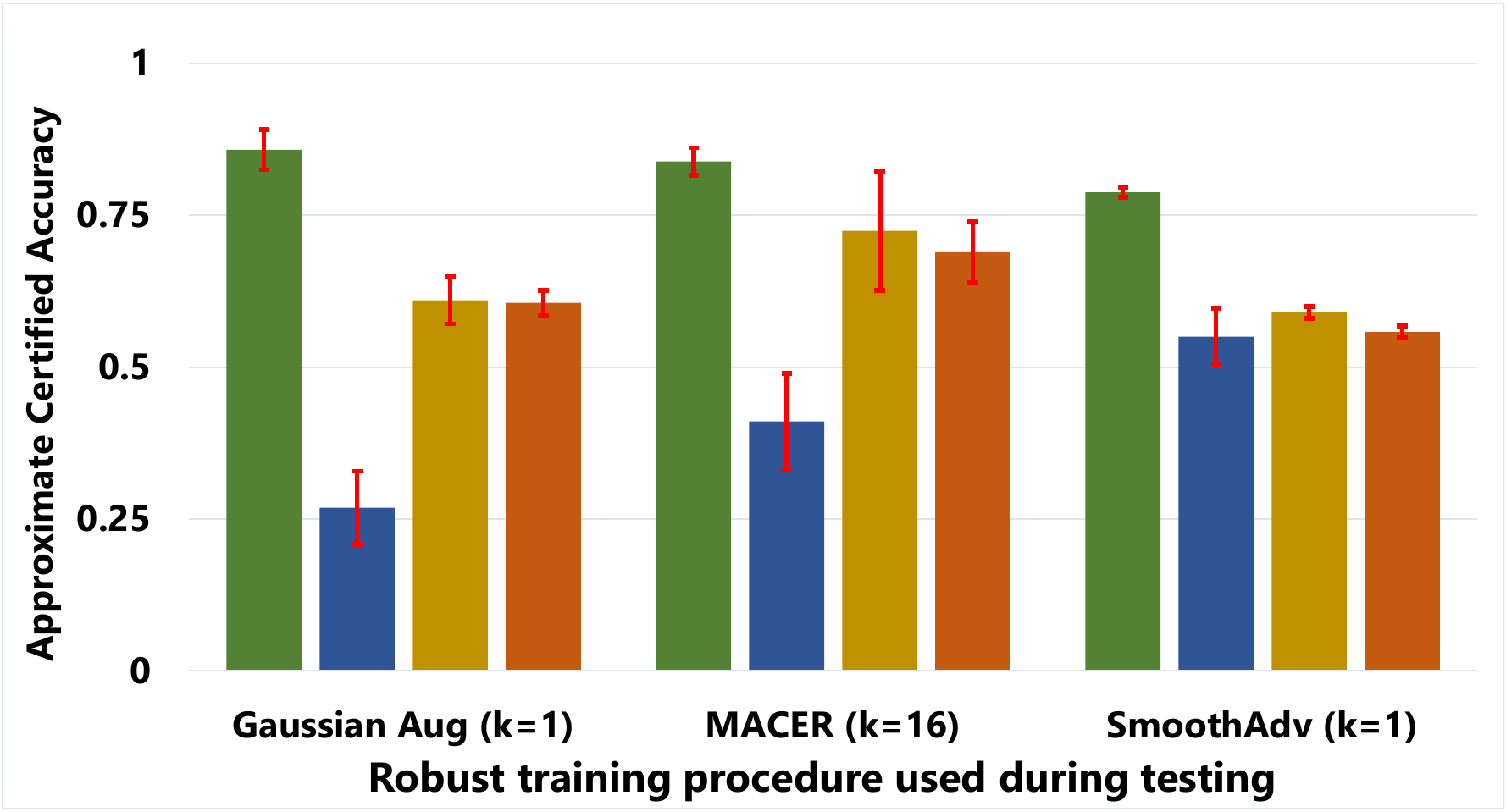}}
  \includegraphics[width=0.9\textwidth]{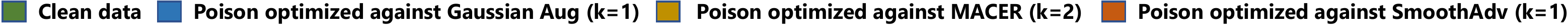}
  }
  \caption{Successful transferability of our poisoned data when victim uses a training procedure different than the one used by the attacker to optimize the poison data. $\sigma = 0.5$ and $\sigma = 0.25$ are used for smoothing during training and evaluation for MNIST and CIFAR10.
  %(\JH{Let's not use ``approximate'' since it's apparent. Use average instead}) other papers have used approximate too.
  %(\JH{Also, can we make it more clear which is method is use for poisoning and which is used for testing?})
  }
  \label{fig:transfer}
\end{figure*}

\begin{table}[tb]
  \caption{Decrease in the mean $\ell_2$-distortion needed to generate adversarial examples using PGD attack against the smooth classifier. This shows that the decision boundary of the smooth classifier is closer to the clean test points of the target class after poisoning.
  }
  \label{Table:empirical_robustness}
  \centering
  \small
  \resizebox{0.75\columnwidth}{!}{
    \begin{tabular}{c|c|c|c}
    \toprule
    & $\sigma$  & Clean data & Poisoned data\\
    \midrule
    \multirow{3}{*}{MNIST} & 0.25 & 3.271$\pm$0.10 & {\bf1.339}$\pm$0.16 \\
    %\cmidrule{2-4}
   & 0.5 & 3.637$\pm$0.15 & {\bf2.170}$\pm$0.09 \\
   %\cmidrule{2-4}
   & 0.75 & 3.961$\pm$0.18 & {\bf2.213}$\pm$0.31 \\
     \midrule
     \multirow{2}{*}{CIFAR10} &0.25 & 1.754$\pm$0.17 & {\bf0.132}$\pm$0.04 \\
    %\cmidrule{2-4}
    &0.5 & 1.996$\pm$0.09 & {\bf0.367}$\pm$0.06 \\
    \bottomrule
    \end{tabular}
    }
\end{table}

\subsection{Poisoning SmoothAdv \cite{salman2019provably}}
Here we present the results of our attack on models trained with SmoothAdv. To yield model with high certified robustness this training method requires the model parameters to be optimized using adversarial training of the smoothed classifier. We used 2 step PGD attack to obtain adversarial example of each point in the batch. We used a single noisy instance of the adversarial example while doing adversarial training. Although, using larger $k$ makes the certified robustness guarantees better, we used $k=1$ to save the computational time required for adversarial training. For bilevel training, we followed the similar procedure to generate adversarial examples for clean and poison data. The adversarial examples are then used as the data for the lower-level problem of Eq.~(\ref{Eq:Bilevel}) to do GA training for optimizing the network parameters. The batch of adversarial examples are recomputed against the updated model after each step of bilevel training. Note that this is an approximation to the actual solution of the minimax problem that has to be solved in the lower-level for generating poison data against SmoothAdv. However, the effectiveness of the attack (results in Table~\ref{Table:smoothadv_attack}) suggests that our approximation works well in practice and certified robustness guarantees achieved from SmoothAdv can be degraded by poisoning. %Unlike the original work \cite{salman2019provably} we do not use other data augmentation techniques like flipping and cropping while generating poison data or during evaluation with SmoothAdv 
%\AM{More details, add about imagenet Does this sound like unfair or weak smoothadv test? can we make it better?}
%\BK{Should we show some examples Clean and generated poisoned train data for these methods on MNIST and CIFAR-10?} Added in the Appendix

\subsection{Effect of the imperceptibility constraint}
%\BK{Title: Effect of Imperceptibility Constraint}
Here we evaluate the effect of using different values of the perturbation strength $\epsilon$ which controls the maximum permissible distortion in Eq.~(\ref{Eq:Bilevel}). We use $\sigma = 0.25$ for smoothing and GA based training to generate and evaluate the attack. The results are summarized in Fig.~\ref{fig:acr_vs_epsilon}, which show that the ACR of the target class decreases as $\epsilon$ increases rendering certification guarantees useless. This is expected since larger $\epsilon$ creates a larger distribution shift among the target class data in the training and the test sets. However, larger permissible distortion to the data make the attack easier to detect by inspection. This is not be desirable from an attacker's perspective who wants to evade detection.
%\BK{This line should go earlier}

\subsection{Transferability of poisoned data}
Here we report the performance of the models trained on the poison data using different training procedures than the one assumed by the attacker for crafting poison data. We used $k=1$ and 2 steps of PGD attack to generate adversarial examples for SmoothAdv and $k=16$ for MACER during retraining. 
The poison data generated against MACER was optimized using $k=2$. 
%However, we observe that this poison data is still able to cause a significant decrease in the robustness guarantees of models trained with all other methods. 
The results are summarized in Fig.~\ref{fig:transfer}, which show that poisoned data optimized against any robust training procedure causes significant reduction in the certified robustness of models trained with a different training methods. Interestingly, poisoned data optimized against GA is extremely effective against other methods, considering the fact that it is the simplest of the three methods. The successful transferability of the poisoned data 
%shows that can reduce the certified robustness guarantees of victim's models trained with state-of-the-art training methods meant to increase certified robustness, 
across different training methods shows how brittle these methods can be when faced with a poisoned dataset. %RS based certified defense methods as its sufficient to create a single poisoned set for transferable attacks. %The significant reduction in average certified radius on both CIFAR10 and MNIST datasets using our poisoned data and transferability of these attacks across training procedures suggest that the attack will be effective on large-scale datasets as well. We plan to explore this in the future by using scalable methods for solving the bilevel optimization problems. 
%This shows that certified defenses can be poisoned importance of the poisoning attacks  impact of data poisoning attacks against certified robustness \JH{is nearly universal} and the role of data in achieving high robustness guarantees. 
%(\JH{Maybe we cite the backdoor attack ... fundamentally broken paper?})
We observe a similar success in transferability of the poison data to models with different architectures (Appendix~\ref{app:transfer_architecture}).

\vspace{-0.2cm}
\subsection{Empirical robustness of poisoned classifiers}\label{sec:exp_emp_robustness}
\vspace{-0.1cm}
Finally, we report the empirical robustness of the smoothed classifier where the base classifier is trained on clean and poisoned data using GA. 
The poisoned data is generated against GA training in the lower-level as in Eq.~(\ref{Eq:Bilevel}). We report the mean $\ell_2$-distortion required to generate an adversarial example using the PGD attack \cite{salman2019provably} against the smoothed classifier using 200 and 100 randomly sampled test points of the target class from MNIST and CIFAR10, respectively, in Table~\ref{Table:empirical_robustness}. 
We observe that our poisoning leads to a decrease in the empirical robustness of the smoothed classifier on clean test data. % similar to the case with linear classifiers.
This backs up our hypothesis
%\JH{suggests$\to$backs up our hypothesis}
that the decision boundary of the smooth classifier must be changed to reduce the certified radius in nonlinear classifiers, similar to linear classifiers (Fig.~\ref{fig:linear}).
%, where we proved that the decision boundary must move closer to the clean test data.% when the perturbation is small.
%\JH{in nonlinear classifiers, similar to linear classifiers (Fig.~\ref{fig:linear}) where we proved that the decision boundary must move closer to the clean test data when the perturbation is small.}

%where data poisoning led to change in the decision boundary to decrease the average certified radius, we observe the similar behavior of poisoning in neural networks. The decrease in the mean distortion of successful attack against the smoothed classifier suggests the decision boundary of the smoothed classifier is closer to the test points of the target class after poisoning. The empirical robustness of the base model being relatively unchanged shows that the decision boundary of the smoothed classifier must be affected to reduce the certified radius.
\if0
\begin{figure}[tb]
  \centering{
  \subfigure[]{\includegraphics[width=0.9\columnwidth]{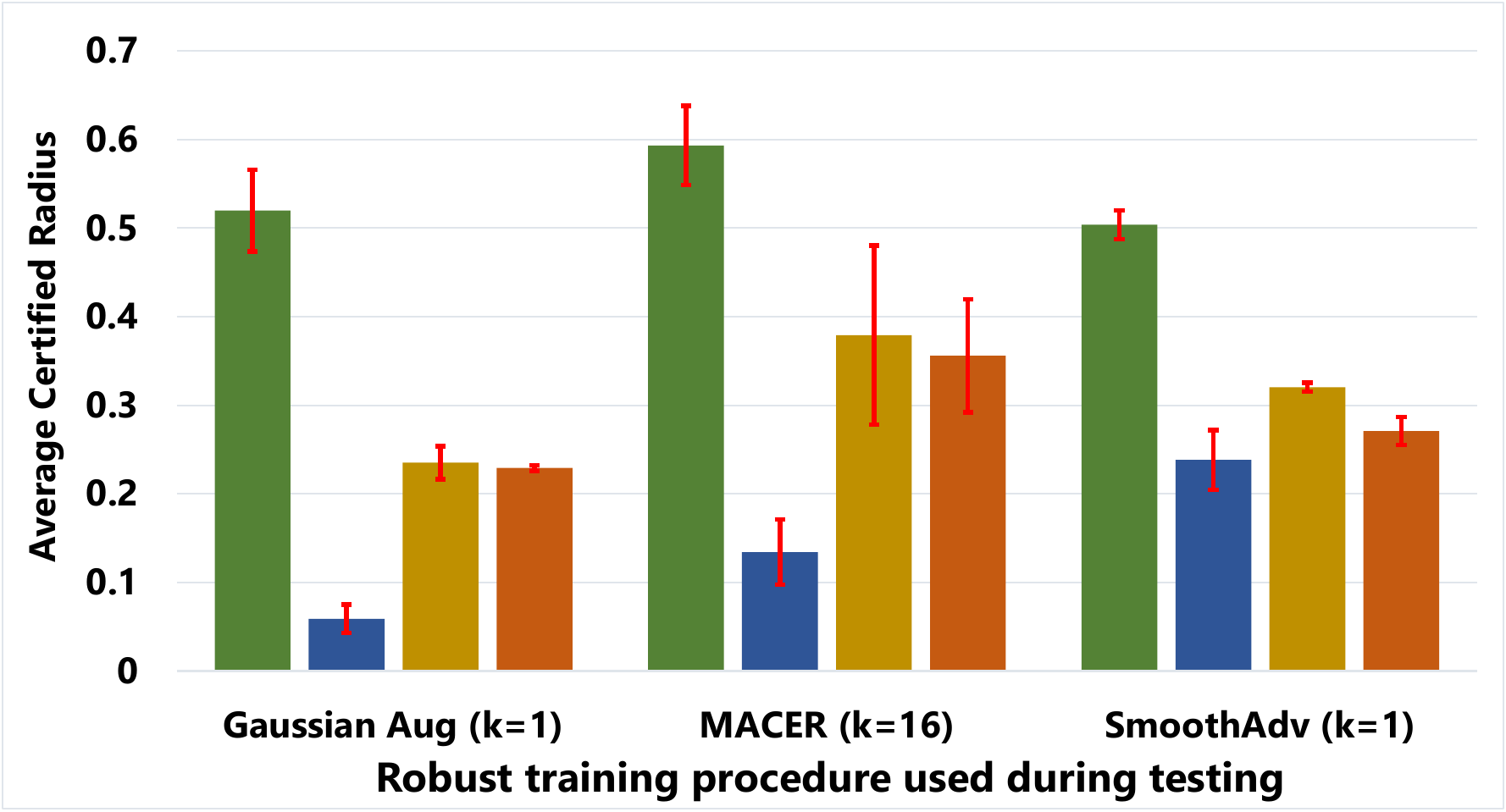}}
  \subfigure[]{\includegraphics[width=0.9\columnwidth]{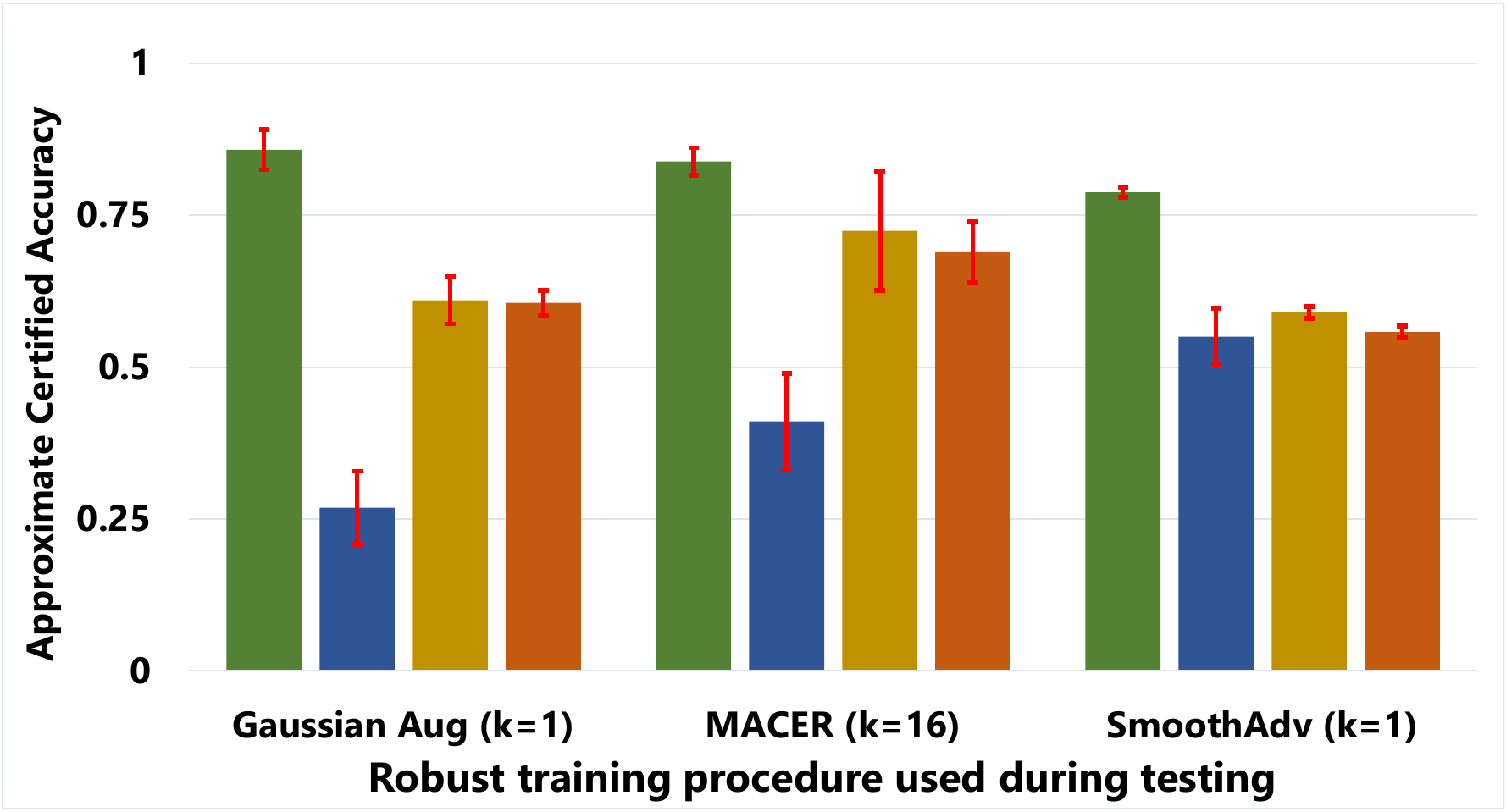}}
  }
  \caption{Transferability of the attack points when victim uses a different training procedure than used by the attacker to generate poison data for CIFAR10 with class ``Ship'' being the target class. $\sigma = 0.25$ is used for smoothing during training and evaluation.}
  \label{fig:transfer_cifar10}
\end{figure}
\fi

\vspace{-0.1cm}
\section{Conclusion}
\vspace{-0.15cm}
Certified robustness has emerged as a gold standard to gauge with certainty the susceptibility of machine learning models to test-time attacks. 
In this work, we showed that these guarantees can be rendered ineffective by our bilevel optimization based data poisoning attack that adds imperceptible perturbations to the points of the target class.
%and ensures high accuracy of the poisoned models on clean data (\JH{Is this true anymore? Did we show any result?}), making the attack difficult to detect.
Unlike previous data poisoning attacks, our attack can reduce the ACR of an entire target class and is even effective against models trained using training methods that have been shown to improve certified robustness. Our results suggests that data quality is a crucial factor in achieving high certified robustness guarantees but is overlooked by current approaches.
%\PY{, which is shown to be a critical but overlooked factor for state-of-the-art test-time certified defenses.}

%While data poisoning attacks to reduce the generalization performance of machine learning models have seen limited success, in this work we proposed a clean label data poisoning attack that affects the certified robustness guarantees of the models trained on it and can withstand retraining and Gaussian augmentation based training. The attack ensures the accuracy of the models trained on poisoned data is comparable to that achieved using clean data making the attack harder to anticipate. In future works we plan investigate the effect of reducing certified radius and its connection to empirical robustness. 
%({\bf JH} Akshay, let's try to reflect Pin-Yu's comment more closely in the text. Maybe in the abstract?)
%\PY{Our results suggest that the (training) data quality and curation are critical and necessary for obtaining sufficient gains form certified defense in test time, which is an overlooked factor in certified defenses. }

%\bibliographystyle{plain}
%\bibliography{iclr2020_conference}
\vspace{-0.15cm}
\section{Acknowledgments}
\vspace{-0.15cm}
This work was supported by the NSF EPSCoR-Louisiana Materials Design Alliance (LAMDA) program \#OIA-1946231 and by LLNL Laboratory Directed Research and Development project 20-ER-014 (LLNL-CONF-817233). 
This work was performed under the auspices of the U.S. Department of Energy by the Lawrence Livermore National Laboratory under Contract No. DE-AC52-07NA27344, Lawrence Livermore National Security, LLC.\footnote{The views and opinions of the authors do not necessarily reflect those of the U.S. government or Lawrence Livermore National Security, LLC neither of whom nor any of their employees make any endorsements, express or implied warranties or representations or assume any legal liability or responsibility for the accuracy, completeness, or usefulness of the information contained herein.} 

%This work was performed under the auspices of the U.S. Department of Energy by the Lawrence Livermore National Laboratory under Contract No. DE-AC52-07NA27344, Lawrence Livermore National Security, LLC\footnote{The views and opinions of the authors do not necessarily reflect those of the U.S. government or Lawrence Livermore National Security, LLC neither of whom nor any of their employees make any endorsements, express or implied warranties or representations or assume any legal liability or responsibility for the accuracy, completeness, or usefulness of the information contained herein.}. 
%This work was supported by LLNL Laboratory Directed Research and Development project 20-ER-014 and released with LLNL tracking number LLNL-CONF-817233.

\clearpage

{\small
\bibliographystyle{ieee_fullname}
\bibliography{egbib}
}

\clearpage
\appendix
%%%%%%%%%%%%%%%%%%%%%%%%%%%%%%%%%%%%%%%%%%%%%%%%%%%%%%%%%%%%%%%%%%%%%%%%%%%%%%%%%%%%%%%%%
\begin{center}
{\LARGE \bf Appendix}
\end{center}
We present the proof of Theorem~\ref{thm:linear} and Corollary~\ref{cor:partial_poisoning} in App.~\ref{app:proofs} followed by a review of bilevel optimization in App.~\ref{app:approxgrad} and our attack algorithm in App.~\ref{app:algorithms}. In App.~\ref{app:additional_experiments}, we present the results of additional experiments on poisoning different classes in the dataset, successful transferability of our poisoning attack to deeper models, performance of the attack when targeting a single test point and effect of using weight regularization on the attack success. We conclude in App.~\ref{app:experiments_details} by providing details of the hyperparameters and models architectures used in the experiments.

\section{Proofs}\label{app:proofs}

\begin{theoremrep}[\ref{thm:linear}]
If the perturbation is large enough, i.e., $\epsilon \geq \frac{\sum_i x_i^{+} - \sum_i x_i^{-}}{n}$ then there are two locally optimal solutions to (\ref{eq:bilevel_linear}) which are $u_i = x_i^{-} - \epsilon$ (Case 1) and $u_i = x_i^{-} + \epsilon$ (Case 2) for $i=1,...,n$. 
Otherwise, the is a unique globally optimal solution which is $u_i = x_i^{-} - \epsilon$ (Case 1) for $i=1,...,n$.
\end{theoremrep}

\begin{proof}
Let $t=-\frac{b}{w}$ be the threshold of the linear classifier.
Also let $\Phi(t):=\int_{-\infty}^{t} P_{-}(x)\;dx$ and $\Psi(t):=\int_{-\infty}^{t} x P_{-}(x)\;dx$.
There are two cases to consider.
\\
{\bf Case 1 ($w>0$)}:  The upper-level cost function is 
\begin{eqnarray*}
f(t)&=&\int_{-\infty}^{t} (t -x)P_{-}(x)\;dx= t\Phi(t)-\Psi(t)
%&=&(u^T\mathbf{1}+c) \Phi((u^T\mathbf{1}+c)) - \Psi(u^T\mathbf{1}+c)
%=t \Phi(t)-\Psi(t).
\end{eqnarray*}
Note that the range $[-\infty,t]$ is where classification is correct for the test data. (Certified radius is 0 for misclassified points by definition.)

The closed-from solution of the lower-level problem gives us $t=-\frac{b}{w}=\frac{\sum_i u_i + \sum_i x_i^{+}}{2n}$, and therefore the perturbation bound $|u_i-x_i^{-}|\leq \epsilon$ implies 
$ \sum_i x_i^{-} - n\epsilon \leq  \sum_i u_i \leq \sum_i x_i^{-} + n\epsilon$
and therefore 
\[
-\frac{\epsilon}{2}+\frac{\sum_i x_i^{+}+\sum_i x_i^{-}}{2n} \leq t \leq \frac{\epsilon}{2}+\frac{\sum_i x_i^{+}+\sum_i x_i^{-}}{2n}.
\]
Also, the assumption $w>0$ poses another constraint:\\  $w \propto \sum_i x_i^{+} - \sum_i u_i>0$ and therefore \\$t = \frac{\sum_i u_i + \sum_i x_i^{+}}{2n} \leq \frac{\sum_i x_i^{+}}{n}$. 
The upper-level problem is therefore
\begin{eqnarray*}
&&\min_{t}\; f(t)=t \Phi(t) - \Psi(t) \;\;\;\;\mathrm{s.t.}\\
&&\;\; -\frac{\epsilon}{2}+\frac{\sum_i x_i^{+}+\sum_i x_i^{-}}{2n} \leq t \leq \frac{\epsilon}{2}+\frac{\sum_i x_i^{+}+\sum_i x_i^{-}}{2n} \\
&&\;\;\mathrm{and}\;\;t\leq\frac{\sum_i x_i^{+}}{n}.
\end{eqnarray*}
% By the way, $w^* =\frac{\sum x_i - \sum u_i}{2n var[x_1,...,x_n,u_1,...,u_n]}
% where var[...] = [\frac{\sum x_i^2+\sum u_i^2}{2n}-\left(\frac{\sum x_i + \sum u_i}{2n}\right)^2]
Since $f$ is non-decreasing (i.e.,$f'(t)=\Phi(t) + t P_{-}(t) - t P_{-}(t) \geq 0$), 
%and convex ($f''(v) = P_{-}(t)\geq 0$),
the minimum is achieved at the left-most boundary $t=-\frac{\epsilon}{2}+\frac{\sum_i x_i^{+}+\sum_i x_i^{-}}{2n}$
which corresponds to $u_i=x_i^{-}-\epsilon,\;\;i=1,...,n$.% The minimum value is $(-\frac{\epsilon}{2} + \sum_i x_i^{pos}}{n})\Phi(-\frac{\epsilon}{2} + \sum_i x_i^{pos}}{n})-\Psi(-\frac{\epsilon}{2} + \sum_i x_i^{pos}}{n})$.
\\

{\bf Case 2 ($w<0$)}:  The upper-level cost function is now 
\[
f(t)=\int_{t}^{\infty} (-t+x)P_{-}(x)\;dx= -t(1-\Phi(t))+(1-\Psi(t)),
\]
which is non-increasing $(i.e., f'(t)=-(1-\Phi)+t P_{-}  - t P_{-}\leq 0)$ and has the constraints:
\[
-\frac{\epsilon}{2}+\frac{\sum_i x_i^{+}+\sum_i x_i^{-}}{2n} \leq t \leq \frac{\epsilon}{2}+\frac{\sum_i x_i^{+}+\sum_i x_i^{-}}{2n}.
\]
and 
\[
t = \frac{\sum_i u_i + \sum_i x_i^{+}}{2n} \geq \frac{\sum_i x_i^{+}}{n}. 
\]
For the solution to be feasible, it is required that
$\frac{\sum_i x_i^{+}}{n} \leq \frac{\epsilon}{2}+\frac{\sum_i x_i^{+}+\sum_i x_i^{-}}{2n}$, that is $\frac{\epsilon}{2} \geq \frac{\sum_i x_i^{+} - \sum_i x_i^{-}}{2n}$ (remember the assumption $\frac{\sum_i x_i^{-}}{n} \leq \frac{\sum_i x_i^{+}}{n}$). 
Therefore if the perturbation is large enough, i.e., $\epsilon \geq \frac{\sum_i x_i^{+} - \sum_i x_i^{-}}{n}$ holds, then the minimum is achieved at the right-most boundary $t=\frac{\epsilon}{2}+\frac{\sum_i x_i^{+}+\sum_i x_i^{-}}{2n}$
which corresponds to $u_i=x_i^{-} + \epsilon,\;\;i=1,...,n$.
\end{proof}
\if0
For simplicity, let $n=1$ (i.e. we have 1 point belonging to each class in our dataset), then $w^* = \frac{-2}{u - x_{+}}, b^*=\frac{u + x_{+}}{u-x_{+}}$. If $u = x_{-}$, we get the parameters of the classifier trained on clean data. Certified radius of the point $x$ under the clean classifier is $CR_{clean}(x) = |x-0.5 (x_{-}+x_{+})|$ and under the poisoned classifier is $CR_{poisoned}(x) = |x - 0.5 (u + x_{+})|$. Replacing $u = x_{-} + \delta$, we have $CR_{poisoned}(x) = |x - 0.5 (x_{-} + x_{+}) - 0.5 \delta|$. Thus, the maximum decrease in the certified radius for a correctly classified point with label -1 (like $x_{-}^{test}$ in the upper-level) is achieved when $\delta = -\epsilon$, implying the point $x_{-}$ is changed to $u = x_{-} + \delta = x_{-} - \epsilon$. However, this change may lead to misclassification of the point $x^{test}_{-}$, thus the optimal distortion for the point $x_{-}$ is $\delta = \min(-|x - 0.5 (x_{-} + x_{+})|, -\epsilon)$.

For $n>1$:
\fi

\begin{corollary}\label{cor:partial_poisoning}%\AM{New}
If the perturbation is large enough, i.e., $\epsilon \geq \frac{\sum_i x_i^{+} - \sum_i x_i^{-}}{n}$ then the reduction in the ACR of the target class, by poisoning an $\alpha$ portion ($\alpha \in [0,1]$) of the target class, with maximum perturbation $\Tilde{\epsilon}$ using Eq. (\ref{eq:bilevel_linear}) is same as that achieved by poisoning the entire target class with $\epsilon = \alpha \Tilde{\epsilon}$.
\end{corollary}
\begin{proof}
To obtain the effect of poisoning an $\alpha$ portion ($\alpha \in [0,1]$) of the target class, with maximum perturbation $\Tilde{\epsilon}$, we follow the proof of Theorem~\ref{thm:linear}, with the change that the solution to the lower-level problem is $t = -\frac{b}{w}=\frac{\sum_{i=0}^{\alpha n} u_i + \sum_{i=0}^{(1-\alpha) n} x_i^{-} + \sum_i x_i^{+}}{2n}$. \\
\\
Thus the solutions to the upper-level problem are \\ 
$t = \frac{-\alpha \Tilde{\epsilon}}{2} + \frac{\sum_i x_i^{+} + \sum_{i=0}^{\alpha n} x_i^{-} + \sum_{i=0}^{(1-\alpha) n} x_i^{-}}{2n}$
which corresponds to $u_i = x_i^{-} - \Tilde{\epsilon},\;\;i=1,...,\alpha n$ for Case 1 and \\\\
$t = \frac{\alpha \Tilde{\epsilon}}{2} + \frac{\sum_i x_i^{+} + \sum_{i=0}^{\alpha n} x_i^{-} + \sum_{i=0}^{(1-\alpha) n} x_i^{-}}{2n}$
which corresponds to $u_i = x_i^{-} + \Tilde{\epsilon},\;\;i=1,...,\alpha n$ for Case 2.
\\\\
The decision boundaries \\\\
$t = \frac{-\alpha \Tilde{\epsilon}}{2} + \frac{\sum_i x_i^{+} + \sum_{i} x_i^{-}}{2n}$
for Case 1 and \\\\
$t = \frac{\alpha \Tilde{\epsilon}}{2} + \frac{\sum_i x_i^{+} + \sum_{i} x_i^{-}}{2n}$ for Case 2 are the same boundaries as obtained by poisoning all points from the target class ($x_i^{-})$ with $\epsilon = \alpha \Tilde{\epsilon}$

\end{proof}

\section{Review of bilevel optimization}\label{app:approxgrad}
A bilevel optimization problem is of the form $\min_{u \in \mathcal{U}} \xi(u,v^*)\;\mathrm{s.t.}\;v^* = \arg\min_{v\in \mathcal{V}(u)}\;\zeta(u,v)$, where the upper-level problem is a minimization problem with $v$ constrained to be the optimal solution to the lower-level problem. General bilevel problems are difficult to solve but if the solution to the lower-level problem can be computed in closed form then we can replace the lower-level problem with its solution, reducing the bilevel problem into a single level problem. We can then use the gradient-based methods to solve the single level problem. The total derivative $\frac{d\xi}{du}(u, v^*(u))$ (hypergradient) using the chain rule is 
\[\frac{d\xi}{du} = \nabla_u \xi + \frac{dv}{du}\cdot\nabla_v \xi.\]
Since ${\nabla_v \zeta=0}$ at ${v=v^*(u)}$ and assuming $\nabla^2_{vv} \zeta$ is invertible we can compute ${\frac{dv}{du}}$ using the implicit function theorem (this can be done even if the solution to lower-level problem can't be found in closed form) which gives \[{\frac{dv}{du} = -\nabla^2_{uv} \zeta (\nabla^2_{vv} \zeta)^{-1}}.\] Thus the hypergradient is
\[\frac{d\xi}{du} = \nabla_u \xi -\nabla^2_{uv} \zeta (\nabla^2_{vv} \zeta)^{-1} \nabla_v \xi \;\mathrm{at} \;(u,v^*(u)).\] Since computation of $(\nabla^2_{vv} \zeta)^{-1}$ is difficult, \cite{domke2012generic,pedregosa2016hyperparameter} proposed to instead approximate the solution to $q = (\nabla^2_{vv} \zeta)^{-1} \nabla_v \xi$ by approximately solving the linear system of equations $\nabla^2_{vv}\cdot q \approx \nabla_v \xi$. This can be done by minimizing $\|\nabla^2_{vv} \zeta\cdot q - \nabla_v \xi\|$ using any iterative solver. Other methods for solving the bilevel optimization problems include using forward/reverse mode differentiation \cite{franceschi2017forward,maclaurin2015gradient,shaban2018truncated} to approximate the inverse and penalty method \cite{mehra2019penalty} to solve the single level problem as a constrained minimization problem. 
\begin{algorithm}%[H] 
\caption{Algorithm for ApproxGrad} 
\label{alg:approxgrad}
{Input}: $\xi, \zeta, M, T_1, T_2, \epsilon, u_{base}, \{\tau_m=0.1\}, \\ \{\rho_{m,t_1}=0.001\},\{\beta_{m,t_2}=0.001\}$ \\
{Output}: $(u_{K})$\\
Initialize $u_0,v_0$ randomly\\
{Begin}
\begin{algorithmic}
\FOR{$m=0,\;\cdots\;,M\textrm{-}1$}
    \STATE{}
    \STATE{\# Approximately solve the lower-level problem }
	\FOR{$t=0,\cdots,T_1\textrm{-}1$}
        \STATE{$v_{t+1} \leftarrow v_{t} - \rho_{m,t_1} \nabla_v \zeta$}
    \ENDFOR
    \STATE{}
    \STATE{\# Approximately solve the linear system}
    \STATE{\# $\nabla^2_{vv} \zeta\cdot q_k = \nabla_v \xi$}
    \FOR{$t=0,\cdots,T_2\textrm{-}1$}
        \STATE{$q_{t+1} \leftarrow q_{t} - \beta_{m,t_2}\nabla_{q}(\|\nabla^2_{vv} \zeta\cdot q_m - \nabla_v \xi\|)$}
    \ENDFOR
    \STATE{}
    \STATE{\# Compute the approximate Hypergradient}
    \STATE{$p_m = \nabla_u \xi - \nabla^2_{uv}\zeta \cdot q_{T_{2}}$}
    
    \STATE{}
    \STATE{\# Update $u_m$ and use projection for the}
    \STATE{\# upper-level constraint}
    \STATE{$u_{m+1} = P(u_{m} - \tau_m p_m, \epsilon, u_{base})$}
    \STATE{}
\ENDFOR
\end{algorithmic}
\end{algorithm}

\section{Attack algorithm}\label{app:algorithms}
Alg.~\ref{alg:main} shows the complete algorithm used to generate the poisoning attack when RS is used for certification and models are trained using GA. The algorithm relies on ApproxGrad (Alg.~\ref{alg:approxgrad}) to solve the bilevel optimization problem. The upper-level cost is a differentiable function that approximates the certified radius of the hard smooth classifier using a soft smooth classifier. The hyperparameter $\alpha$ is the inverse temperature parameter of softmax. As $\alpha \rightarrow \infty$, softmax converges to argmax almost everywhere. As a result $\tilde{g}_{\theta}$ converges to $g_{\theta}$ almost everywhere and thus soft randomized smoothing converges to hard randomized smoothing almost everywhere. Although, in this work we considered RS as the procedure for certification (due to its scalability to large models and datasets), any other certification procedure can be used as the upper-level cost as long as its differentiable. Moreover, Alg.~\ref{alg:main} uses $\mathcal{L}_{\mathrm{GaussAug}}$ in the lower-level to train the model, but like the case with upper-level cost any other loss function can be used to obtain the model parameters. This flexibility of our method allows us to generate poison data against MACER and SmoothAdv using their loss functions in the lower-level.

\subsection{ApproxGrad}
For an unconstrained bilevel problem of the form $\min_{u} \xi(u,v^*)\;\mathrm{s.t.}\;v^* = \arg\min_{v}\;\zeta(u,v)$, if $\zeta(u,v)$ is strongly convex then we can replace the lower-level problem with its necessary condition for optimality and write the bilevel problem as the following single level problem $\min_{u} \xi(u,v^*)\;\mathrm{s.t.}\;\nabla_v \zeta(u,v) = 0$. Assuming $\nabla^2_{vv} \zeta$ is invertible everywhere we can compute the hypergradient at the point $(u, v^*(u))$ as $\frac{d\xi}{du} = \nabla_u \xi -\nabla^2_{uv} \zeta (\nabla^2_{vv} \zeta)^{-1} \nabla_v \xi$.

The ApproxGrad algorithm approximates the Hessian-inverse vector product by approximately solving a system of linear equation using an iterative solver such as gradient descent or conjugate gradient method. In this work we use Adam optimizer to solve this system. Since our problem for data poisoning in Eq.~(\ref{Eq:Bilevel}) involves a constraint in the upper-level we use projection to enforce the constraint. The full algorithm for solving the bilevel optimization problem using ApproxGrad is present in Alg.~\ref{alg:approxgrad}. For our attack the lower-level problem involves a deep neural network, which can have multiple local minima and thus optimizing against a single local minima in the bilevel problem is not ideal. To overcome this problem we reinitialize the lower-level variable $v$ after few upper-level iterations to prevent the poisoning points from overfitting to a particular local minima. Empirically, this helps us find poisoning points that remain effective even after the model is retrained from scratch making them generalize to different initialization of the neural network. 

\begin{figure*}
\small
\centering
\subfigure[Clean data (odd numbered rows) and poisoned data generated by our attack (even numbered rows) for all digits in MNIST]
{
\includegraphics[width=0.98\columnwidth,height=15cm]{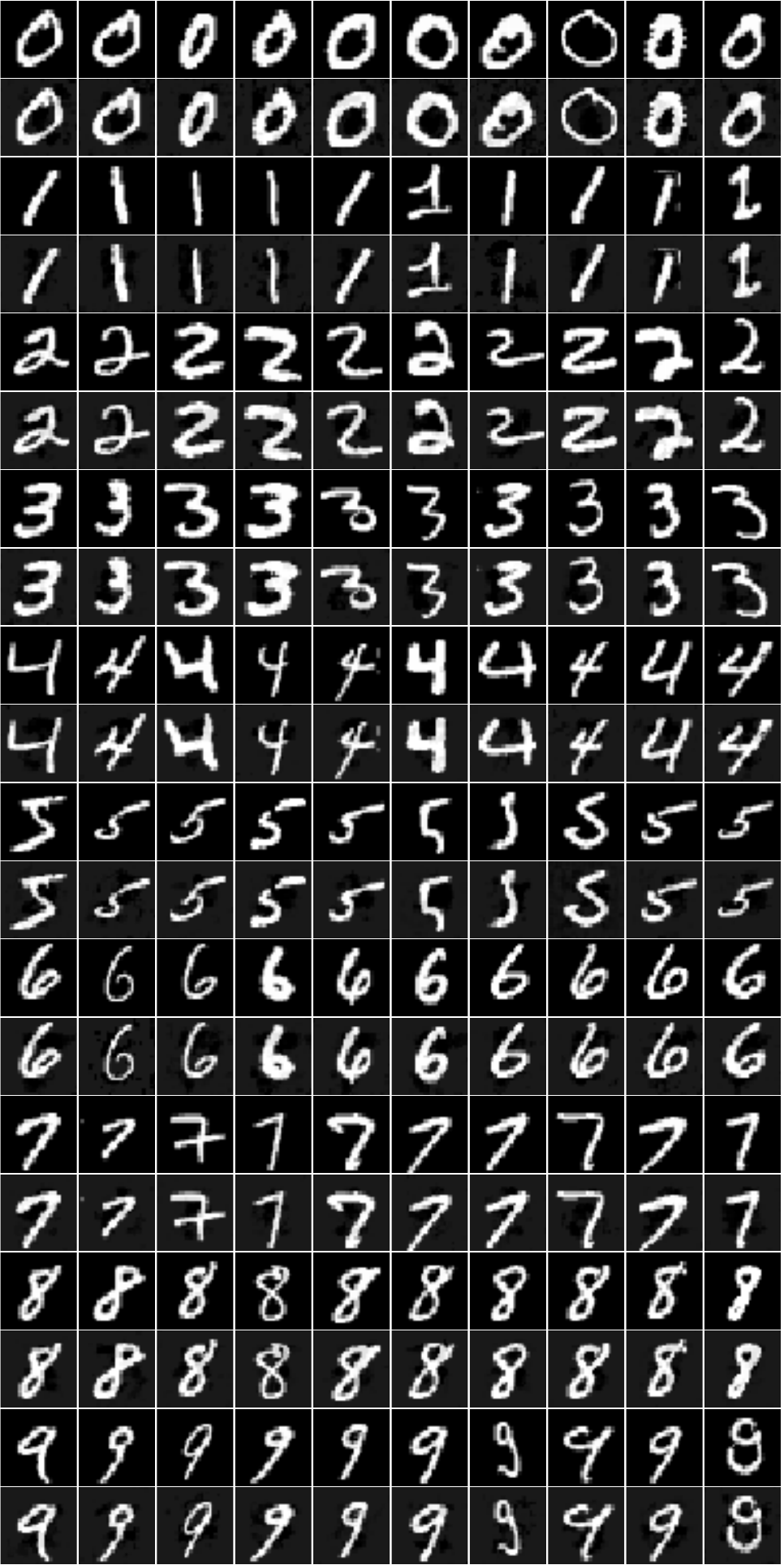}
}\hfill
\subfigure[Clean data (odd numbered rows) and poisoned data generated by our attack (even numbered rows) for all classes of CIFAR10]
{
\includegraphics[width=0.98\columnwidth,height=15cm]{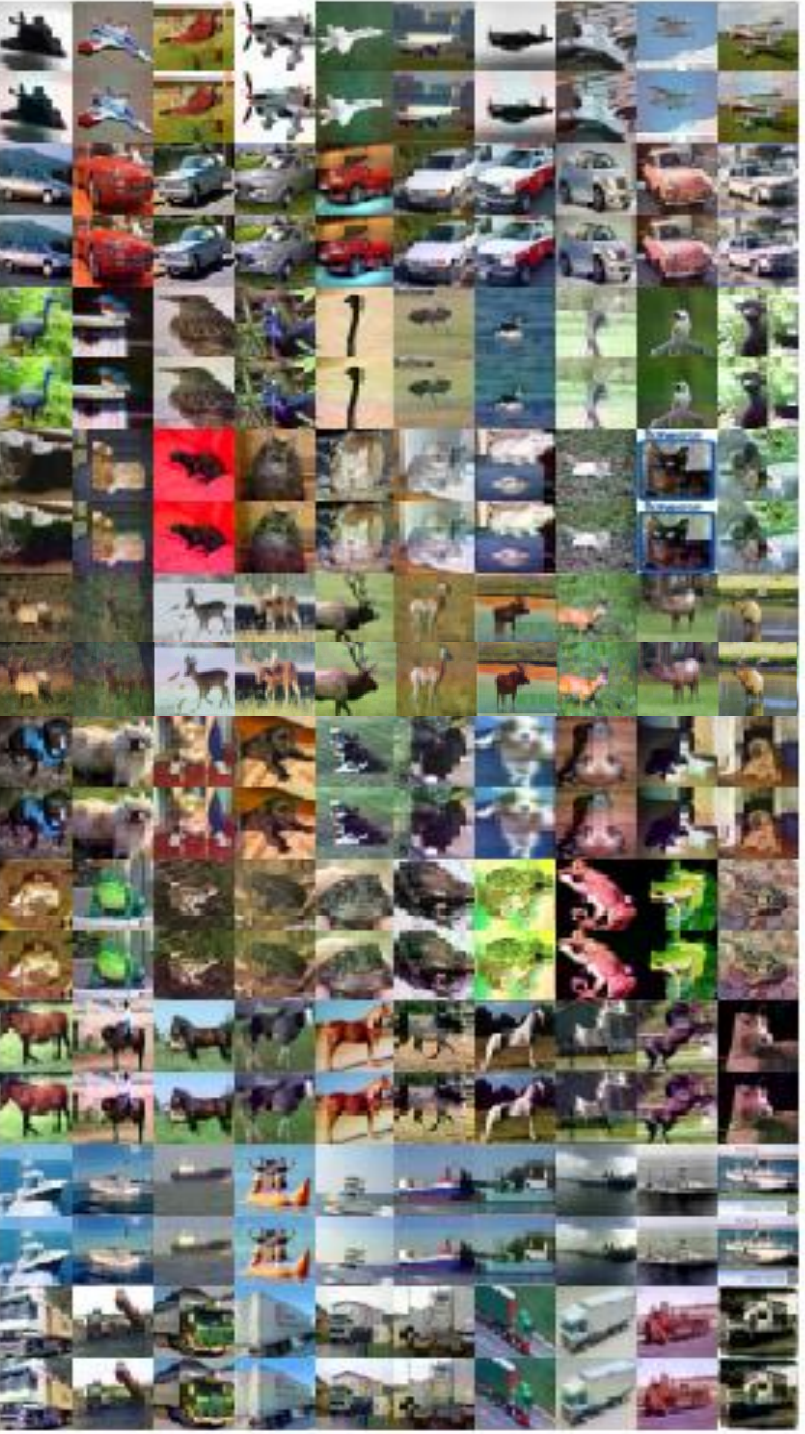}
}
\caption{Imperceptibly distorted poison data generated by our algorithm against Gaussian augmented training which causes a significant reduction in the certified robustness guarantees of the models. The average certified radius and certified accuracy of models trained on clean and poisoned data are reported in App.~\ref{app:all_classes} and Fig.~\ref{fig:all_classes}.}
\label{fig:attack_egs}
\end{figure*}

\begin{figure*}[tb]
  \centering{
  \subfigure[Average certified radius of all digits in MNIST]{\includegraphics[width=\columnwidth]{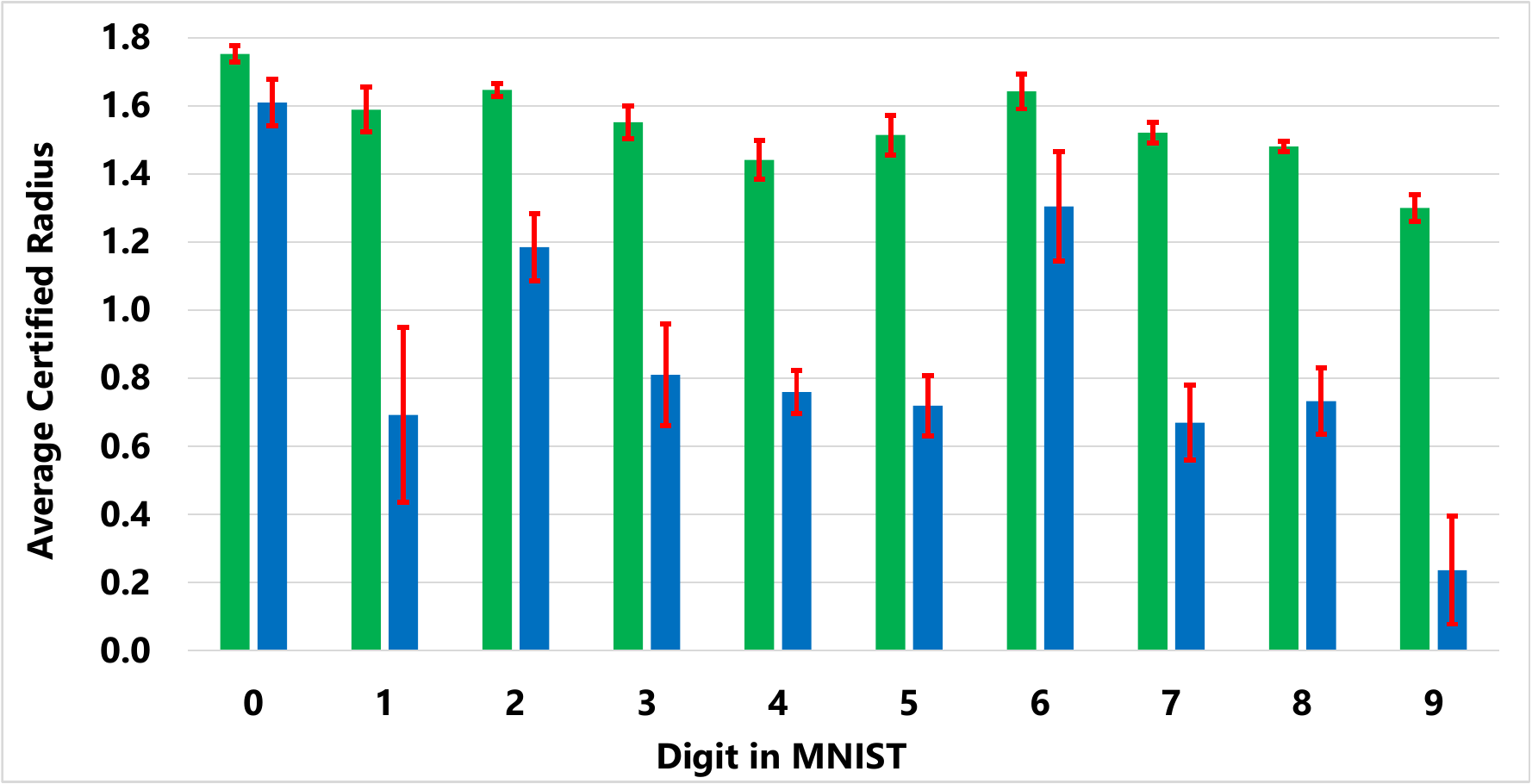}}
  \hspace{0.15in}
  \subfigure[Approximate certified accuracy of all digits in MNIST]{\includegraphics[width=\columnwidth]{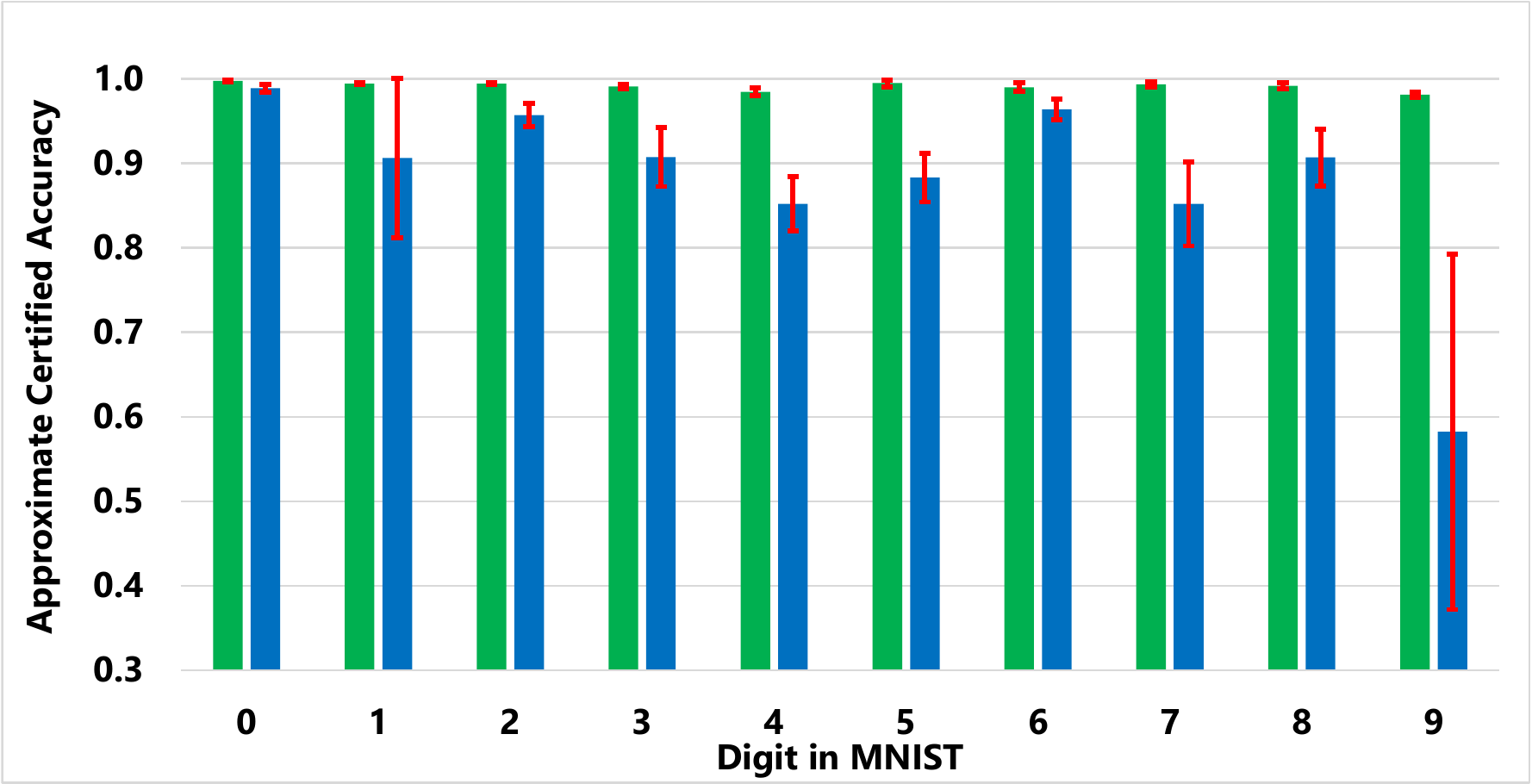}}
  \subfigure[Average certified radius of all classes in CIFAR10]{\includegraphics[width=\columnwidth]{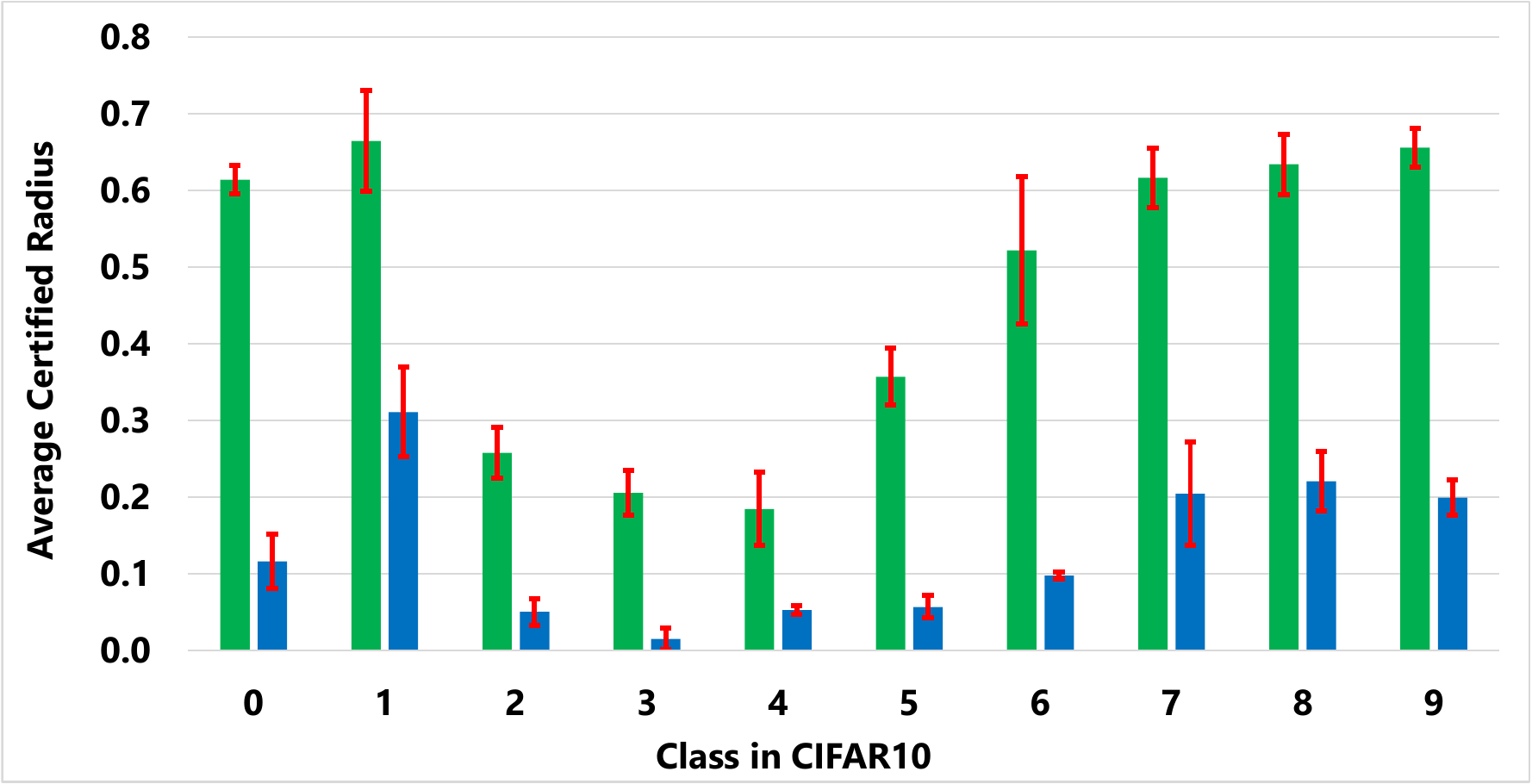}}
  \hspace{0.15in}
  \subfigure[Approximate certified accuracy of all classes in CIFAR10]{\includegraphics[width=\columnwidth]{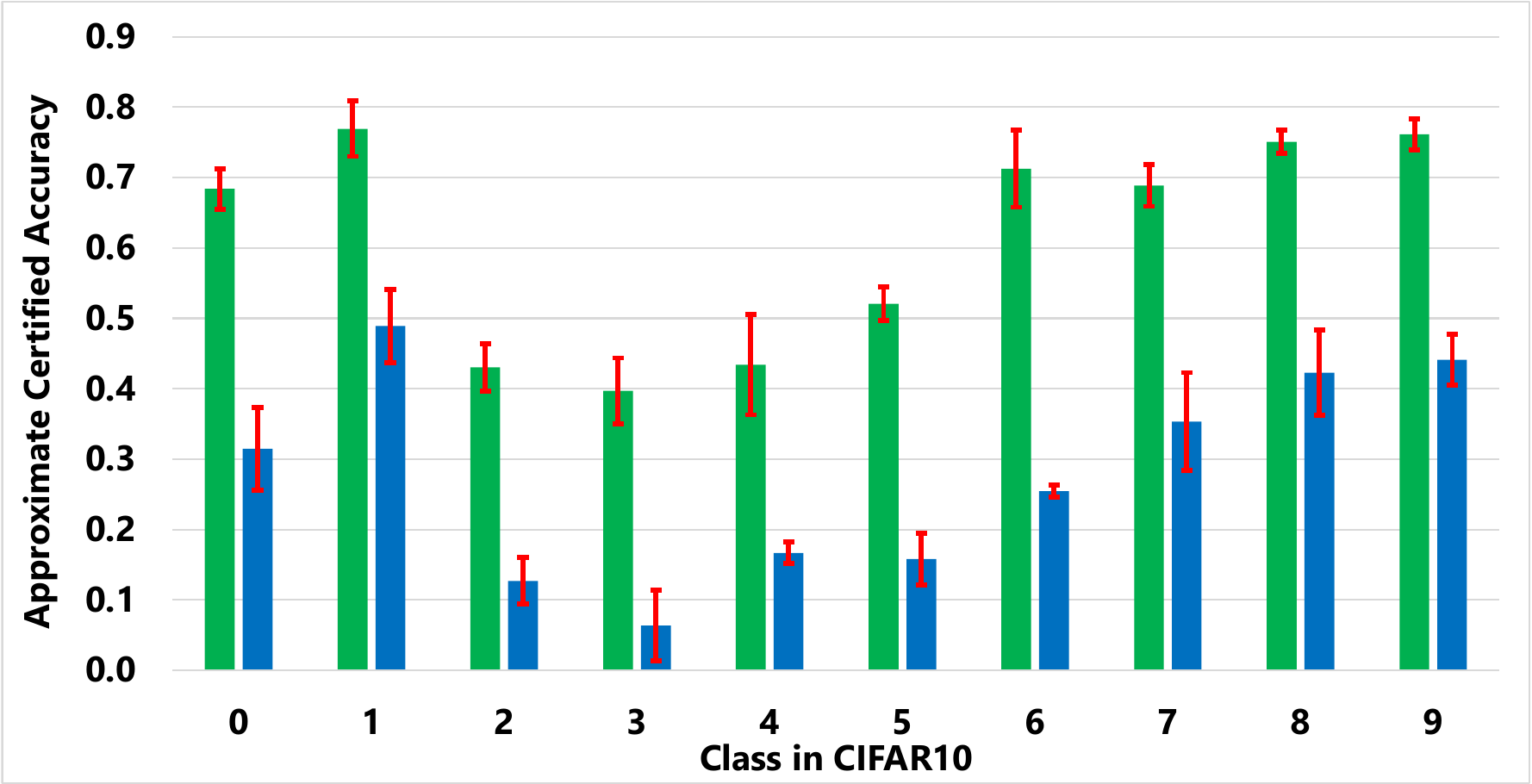}}
  \includegraphics[width=0.5\columnwidth]{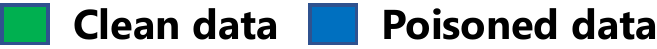}
  }
  \caption{Successful poisoning attack against all classes in MNIST and CIFAR10 dataset.
  }
  \label{fig:all_classes}
\end{figure*}
 
%%%%%%%%%%%%%%%%%%%%%%%%%%%%%%%%%%%%%%%%%%%%%%%%%%%%%%%%%%%%%%%%%%%%%%%%%%%%%%%%%%%%%%%%%%%%%%%%%%%%%%%%
\section{Additional experiments}\label{app:additional_experiments}
%%%%%%%%%%%%%%%%%%%%%%%%%%%%%%%%%%%%%%%%%%%%%%%%%%%%%%%%%%%%%%%%%%%%%%%%%%%%%%%%%%%%%%%%%%%%%%%%%%%%%%%
\subsection{Comparison with standard data poisoning}\label{app:standard_acc}
The standard data poisoning attack creates poison data so that the accuracy of the victim's model trained on it is significantly lower than the accuracy attainable with training on clean data. The bilevel optimization problem for this attack is as follows.
\begin{equation}
    \begin{split}
        &\min_{u}\;\; \mathcal{L}_\mathrm{standard}(\mathcal{D}^\mathrm{val}) \\
        &\mathrm{s.t.}\;\; \|\delta_i\|_{\infty} \leq \epsilon,\;\;  i=1,...,n,\;\;\mathrm{and} \\ 
        \theta^* = \arg&\min_{\theta}\; \mathcal{L}_\mathrm{standard}(\mathcal{D^\mathrm{clean}} \bigcup \mathcal{D^\mathrm{poison}}; \theta).
    \end{split}
    \label{eq:bilevel_poison_accuracy}
\end{equation}

Here $\mathcal{L}_\mathrm{standard}(\mathcal{D};\theta) = \frac{1}{|\mathcal{D}|}\sum_{(x_i,y_i) \in \mathcal{D}} l_{ce}(x_i,y_i;\theta)$, where $l_{ce}$ is the cross entropy loss. We used this formulation to generate the poisoned dataset for reporting the results in Table~\ref{Table:difficulty_of_poisoning} with $\epsilon = 0.1$ for MNIST and  $\epsilon = 0.03$ for CIFAR10. The attack modifies all the points in the target classes. Specifically, our attack targets misclassification of the digit 8 in MNIST and class ``Ship'' in CIFAR10. The poisoned dataset obtained after solving the bilevel optimization was then used to train five models starting from random initializations with different training procedures. The results of which are reported in Table~\ref{Table:difficulty_of_poisoning}. As expected the models trained with standard training on the poisoned data perform the worst in terms of accuracy since the attack was optimized against standard training. However, the generated poison data has little to no effect when a training procedures that improves certified adversarial robustness is used. This shows that the effect of standard data poisoning can easily be nullified if a victim trains the model with a these training procedures. This gives a false sense of security of the models trained with certified defenses to data poisoning attacks. Thus, in this work we study the effect of poisoning on training procedures meant to improve certified adversarial robustness and show that their guarantees become meaningless when the dataset is poisoned.

\subsection{Isotropic Gaussians}\label{app:isotropic_gaussian}
Here we validate the solution found by solving the bilevel optimization against the analytical solution of a toy problem. Consider a two-dimensional dataset comprising of points drawn from two isotropic Gaussian distributions. Let $\mathbb{P}(x|y=-1)) = \mathcal{N}(\mu_1, \sigma^2I)$ and $\mathbb{P}(x|y=1) = \mathcal{N}(\mu_2, \sigma^2I)$ and equal prior $\mathbb{P}(y=1) = \mathbb{P}(y=-1)$. For a point $x$, the Bayes optimal classifier predicts $y=1$ if $\mathbb{P}(y = 1|x) >= \mathbb{P}(y = -1|x)$ and predicts $y=-1$ otherwise. The decision boundary of the Bayes optimal classifier is given by $(\bf{x} - {\mu_1})^T(\bf{x} - {\mu_1}) = (\bf{x} - {\mu_2})^T(\bf{x} - {\mu_2})$. This is also the decision boundary of the smoothed classifier. Assuming the attacker is poisoning the class with label $-1$ and maximum permissible distortion is $\epsilon$, our analysis showed that maximum reduction in radius occurs if the entire distribution shifts by $\epsilon$ i.e. the new mean of the class with label -1 is $\mu_1 - \epsilon$ and the decision boundary is $(\bf{x} - (\mu_1 - \epsilon))^T(\bf{x} - (\mu_1 - \epsilon)) = (\bf{x} - {\mu_2})^T(\bf{x} - {\mu_2})$. Since the test distribution is unchanged, the ACR for the test points with labels -1 is reduced by $\frac{\epsilon}{\sqrt{2}}$. Using $\mu_1 = 0.2, \mu_2 = 0.8, \sigma_1 = \sigma_2 = 0.3, \epsilon = 0.1$ and using logistic regression in the lower-level, analytically the certified radius must decrease from $0.4243$ to $0.3546$. The solution by solving the bilevel optimization numerically (Table~\ref{Table:isotropic_gaussian}) matches the analytic solution. 
\begin{table}[tb]
  \caption{Degradation of certified adversarial robustness of logistic regression trained with GA on a toy 2-isotropic Gaussians dataset.
  }
  \label{Table:isotropic_gaussian}
  \centering
  \small
  \resizebox{0.85\columnwidth}{!}{
    \begin{tabular}{c|cc|cc}
    \toprule
    \multirow{2}{*}{$\sigma$} & \multicolumn{2}{c|}{\makecell{Certified Robustness on \\ clean data}} &  \multicolumn{2}{c}{\makecell{Certified Robustness on \\ poisoned data}} \\
    %\midrule
    & ACR & ACA(\%) &  ACR & ACA(\%)  \\
    \midrule
    0.25 & 0.4047 & 90.00 & 0.3585 & 88.00 \\
    0.50 & 0.4139 & 90.00 & 0.3587 & 87.60 \\
    0.75 & 0.4123 & 90.00 & 0.3544 & 87.60 \\
    \bottomrule
    \end{tabular}
    }
\end{table}

\begin{figure*}[tb]
  \centering{
  \subfigure[Average certified radius of class ``Ship'' in CIFAR10]{\includegraphics[width=0.85\columnwidth]{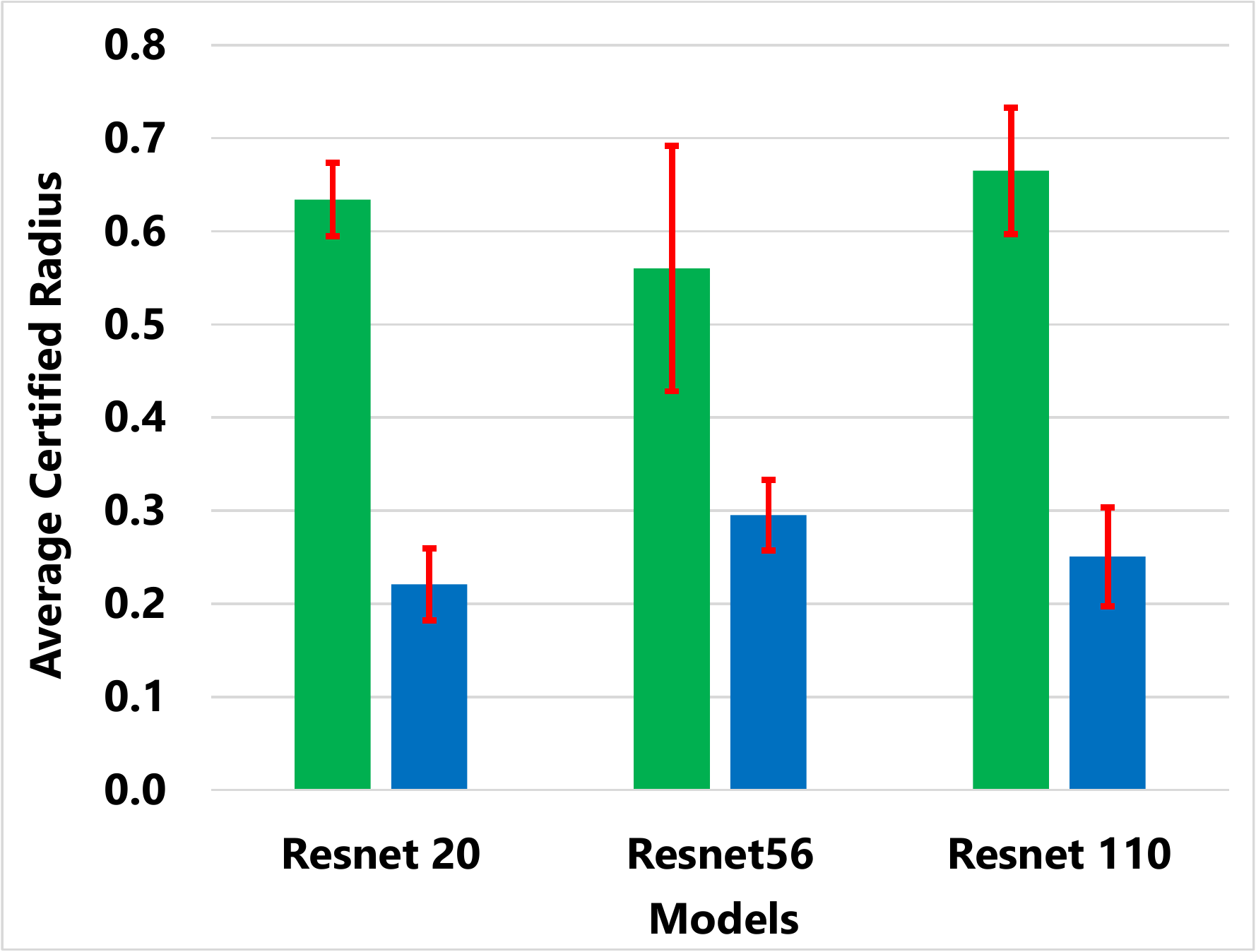}}
  \hspace{0.15in}
  \subfigure[Approximate certified accuracy of class ``Ship'' in CIFAR10]{\includegraphics[width=0.85\columnwidth]{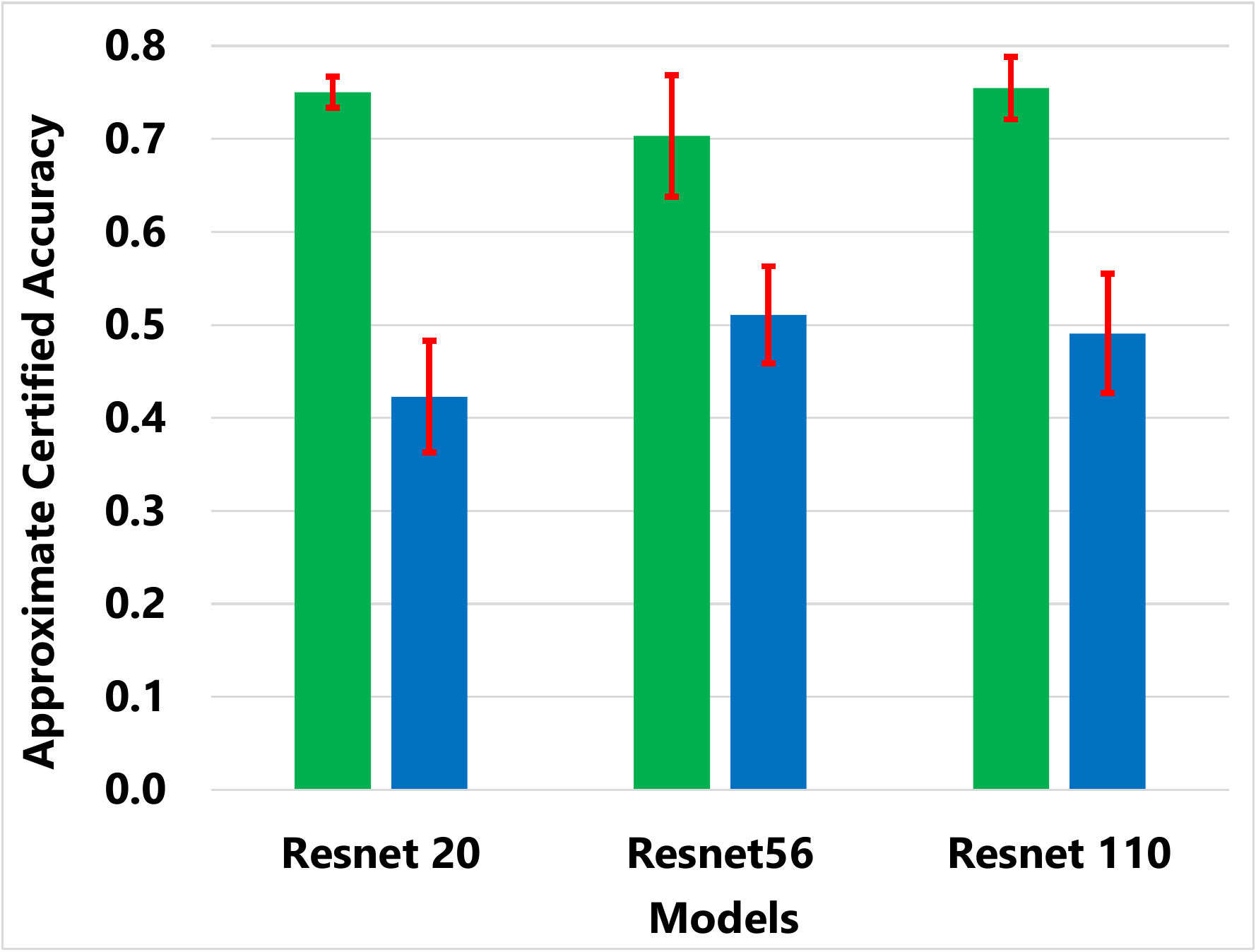}}
  \includegraphics[width=0.5\columnwidth]{images/legend_2_c.pdf}
  }
  \caption{Successful transferability of poisoning attack against deeper models. The poison data is optimized against Resnet-20. GA with $\sigma = 0.5$ is used for poison generation and evaluation.
  }
  \label{fig:transferability_arch}
\end{figure*}

\subsection{Targeting other classes}\label{app:all_classes}
In this section we present the results of our poisoning attack on different target classes where the models are trained using GA during poison generation and evaluation. Since MNIST and CIFAR10 both have 10 classes we create 10 poisoning sets each targeting a particular class. The results of retraining models from five random initializations on each of the 10 poisoning sets are summarized in Fig.~\ref{fig:all_classes}. Reduction in average certified radius for all classes shows that an attacker can generate poison data to affect any class in the dataset.

\subsection{Transferability to different architectures}\label{app:transfer_architecture}
Here we present the results of transferability of the poisoned data generated against Resnet-20 targeting the class ``Ship'' to bigger models. In particular we present the results on Resnet-56 and Resnet-110 \cite{He_2016_CVPR} models in Fig.~\ref{fig:transferability_arch}. As seen from the results the poisoned data generated against Resnet-20 is successful in reducing the certified radius of the target class even if the victim uses a larger model. We report the results of training the models on clean and poisoned data starting from three random initializations and certify using 500 randomly sampled points of the target class from the clean test set. Our results suggest that the poisoned data generated using our procedure are agnostic to the training procedure (Fig.~\ref{fig:transfer}), model (Fig.~\ref{fig:transferability_arch}) and metric (RS or empirical robustness) used by the victim during evaluation highlighting the threat of data poisoning. 

\subsection{Effect of weight regularization}\label{app:weight_reg}
Previous works \cite{carnerero2020regularisation} have shown weight regularization to mitigate the effect of data poisoning attacks. Here, we evaluate the attack success when using different coefficients for weight regularization in standard and GA training.  
Results in Table~\ref{Table:effect_of_weight_regularization} show that our attack significantly reduces the ACR of models, especially those trained without GA or weight regularization.
Similar to previous works \cite{carnerero2020regularisation}, we see that models trained with large regularization (without GA) are difficult to poison. This increased robustness, however comes at the cost of accuracy (target class accuracy of models trained with clean data drops from 99\% to 92\%), suggesting a trade-off. On the other hand, for models trained with GA, using large regularization leads to a significant drop in their certified robustness guarantees, even without poisoning (ACR for models trained on clean data drops from 1.48 to 0.85), there by making large regularization undesirable to use with GA. 
This shows that our attack remains quite effective even when different amounts of weight regularization are used during model retraining.  
\vspace{-0.15cm}
\begin{table}[h]
    \caption{%\AM{New}
    Effect of weight regularization and Gaussian data augmentation ($\sigma$=0.5) on ACR of digit 8 of MNIST ($\epsilon$=0.1). Mean and s.d. of 3 random initializations. Bold entries are reported in Table \ref{Table:cohen_attack}.}
  \label{Table:effect_of_weight_regularization}
  \centering
  \small
  \resizebox{0.95\columnwidth}{!}{
    \begin{tabular}{c|c|c|c|c|c}
    \toprule
   Training & Data & No Reg. & 1E-4 & 1E-2 & 1E-1\\
    \midrule
    \multirow{2}{*}{\makecell{Without \\ GA}} & Clean & 0.95$\pm$0.10 & 0.89$\pm$0.06 & 0.87$\pm$0.11 & 0.82$\pm$0.04  \\
    &Poisoned & 0.01$\pm$0.01 & 0.03$\pm$0.05 & 0.37$\pm$0.06 & 0.68$\pm$0.05 	\\
    \midrule
    \multirow{2}{*}{\makecell{With \\ GA}}&Clean & {\bf1.48$\pm$0.02} & 1.49$\pm$0.03& 1.29$\pm$0.15 & 0.85$\pm$0.09 \\
    &Poisoned & {\bf0.73$\pm$0.10} & 0.62$\pm$0.02 & 0.58$\pm$0.11& 0.72$\pm$0.08\\
    \bottomrule
    \end{tabular}
    }
\end{table}

\subsection{Attack success by poisoning 1\% of the data}\label{app:partial_poisoning}
%\AM{New}
Similar to \cite{huang2020metapoison}, where the effect of poisoning 1\% of the training data is evaluated on a single test point, we randomly select 5 test images of the bird class and generate poison data to reduce their certified radius individually. Using Alg.~\ref{alg:main}, we poison 500 bird images (1\% of CIFAR10) closest to the target, with $\epsilon$=0.06. Using 3 randomly initialized models trained with GA, our attack can reduce the certified radius of 5 targets from 0.63 to 0.26 on average. 

%%%%%%%%%%%%%%%%%%%%%%%%%%%%%%%%%%%%%%%%%%%%%%%%%%%%%%%%%%%%%%%%%%%%%%%%%%%%%%%%%%%%%%%%%%%%%%%%%%%%%%%%
\section{Details of the experiments}\label{app:experiments_details}
%%%%%%%%%%%%%%%%%%%%%%%%%%%%%%%%%%%%%%%%%%%%%%%%%%%%%%%%%%%%%%%%%%%%%%%%%%%%%%%%%%%%%%%%%%%%%%%%%%%%%%%
All codes are written in Python using Tensorflow/Keras, and were run on Intel Xeon(R) W-2123 CPU with 64 GB of RAM and dual NVIDIA TITAN RTX. Implementation and hyperparameters are described  below.

\subsection{Data splits}
For MNIST, we use 55000 points as the training data and 5000 points for validation data. We have roughly 500 points belonging to the target class in the validation set which is used in the upper-level problem of the bilevel optimization presented in Eq.~(\ref{eq:bilevel_simple}). For CIFAR10, we use 45000 points as the training data and 5000 points for validation data. Similar to MNIST we have roughly 500 points belonging to the target class in the validation set. The test sets of both the datasets comprises of 10000 points. We use 500 randomly sampled points of the target class from the test set to report the results of certified and empirical robustness of the models trained on clean and poisoned data. 

\subsection{Model Architecture}
For the experiments on the MNIST dataset, our network consists of a convolution layer with kernel size of 5x5, 20 filters and ReLU activation, followed by a max pooling layer of size 2x2. This is followed by another convolution layer with 5x5 kernel, 50 filters and ReLU activation followed by similar max pooling and dropout layers. Then we have a fully connected layers with ReLU activation of size 500. Lastly, we have a softmax layer with 10 classes. The accuracy of the model on clean data when optimized with the Adam optimizer using a learning rate of 0.001 for 100 epochs with batch size of 200 is 99.3\% (without GA). 
For the experiments on the CIFAR10 dataset, we use the Resnet-20 model. The accuracy of the model on clean data when optimized with the Adam optimizer using a learning rate of 0.001 for 100 epochs with batch size of 200 is 85\% (without GA). For all CIFAR10 experiments except for the experiments with SmoothAdv, we trained the models using data augmentation (random flipping and random cropping).
We used the same parameters for training the models with different robust training procedures on clean and poisoned data.

\subsection{Hyperparameters}
For experiments with MNIST we used $\epsilon = 0.1, K = 20, \alpha = 16$. The batch size used for lower-level training was 1000, of which 100 points belonged to the poisoned set (target class). The batch size for validation set was 100 which only consisted of points from the target class. The lower-level was trained using different training procedures on clean and poisoned data. For experiments with CIFAR10 we used $\epsilon = 0.03, \lambda = 0.06, M = 20, \alpha = 16$. The batch size used for lower-level training was 200, of which 20 points belonged to the poisoned set (target class). The batch size for validation set was 20 which only consisted of points from the target class. For training with GA the lower-level is trained with a single noisy image of the clean and poisoned dataset. The same setting is used while retraining. For generating poison data against MACER the lower-level is trained with $K=2, \lambda=1, \gamma = 8$. During retraining, $K=16, \lambda=16, \gamma=8$ are used for MNIST. For CIFAR10, we use $K=16, \gamma=8$ and $\lambda = 12$ for $\sigma=0.25$ and $\lambda=4$ for $\sigma=0.5$. The hyperparameters during retraining are similar to the ones used in the original work. For Smoothadv, $k=1$ and 2-step PGD attack are used to generate adversarial examples of the smooth classifier. These adversarial examples along with GA are used to do adversarial training during poison generation and retraining. 

In our experiments with generating poisoned data against GA and MACER we used $P = 50, T_1 = T_2 = 10, \tau = 0.1, \rho = 0.001, \beta = 0.01$ for ApproxGrad. We used all the same hyperparameters for SmoothAdv except $T_1 = 1$. For certification we used the CERTIFY procedure of \cite{cohen2019certified}, with $n_0 = 100, n=100000, \alpha = 0.001$. For measuring empirical robustness of the smoothed classifier, we used the mean $\ell_2$ distortion required by PGD attack to generate an adversarial example as done in \cite{salman2019provably}. The attack is optimized for 100 iterations for different values of $\ell_{2}$ distortion between (0.01, 10). We used 20 augmentations for each test point of MNIST and 10 for CIFAR10. To report the results for empirical robustness we record the minimum distortion for a successful attack for each test point.

For the watermarking baseline, we randomly selected an image ($other$) from the classes other than the target class and over-layed them on top of the target class images ($base$) with an opacity of $\gamma = 0.1$ i.e. ($poison\_image =  \gamma \cdot other + (1 - \gamma) \cdot base$). We then clip the images to have $\ell_{\infty}$ distortion of $\epsilon$ to make our bilevel attack comparable in terms of maximum distortion.

\end{document}